\documentclass[10pt,journal,compsoc,twoside]{IEEEtran}
\ifCLASSOPTIONcompsoc
  \usepackage[nocompress]{cite}
\else
  \usepackage{cite}
\fi
\ifCLASSINFOpdf
\else
\fi
\usepackage{mathtools}
\usepackage{array}
\usepackage{float}
\usepackage{hyperref}
\usepackage{graphicx}
\graphicspath{{./graphics/}}
\usepackage{subfigure}
\usepackage{multicol}
\usepackage{multirow}
\usepackage{paralist}
\usepackage{bm}
\usepackage{amsmath,amsthm,amssymb,amsfonts}
\usepackage[usenames,dvipsnames]{color}
\usepackage{dsfont}
\usepackage{url}
\usepackage{rotating}
\usepackage{enumerate}
\usepackage{color}
\usepackage{slashbox}

\usepackage{ifpdf}
\usepackage{cite}

\usepackage{textcomp, booktabs}
\usepackage{colortbl}
\usepackage{gensymb}

\usepackage{times}
\usepackage{epsfig}
\usepackage{caption}
\usepackage{setspace}
\usepackage{threeparttable}

\makeatletter
\newif\if@restonecol
\makeatother

\usepackage[linesnumbered, ruled]{algorithm2e}
\usepackage{algpseudocode}

\newtheorem{proposition}{Proposition}
\newtheorem{definition}{Definition}

\begin{document}
%
\title{Temporal-Spatial Causal Interpretations for Vision-Based Reinforcement Learning}

%
%
%
%

\author{Wenjie~Shi,
        Gao~Huang,~\IEEEmembership{Member,~IEEE,}
        Shiji~Song,~\IEEEmembership{Senior~Member,~IEEE,}
        and~Cheng~Wu 
\IEEEcompsocitemizethanks{
\IEEEcompsocthanksitem W. Shi, S. Song and C. Wu are with the Department of Automation / Beijing National Research Center for Information Science and Technology (BNRist), Tsinghua University, Beijing 100084, China. E-mail: shiwj16@mails.tsinghua.edu.cn; \{shijis, wuc\}@tsinghua.edu.cn.
\IEEEcompsocthanksitem G. Huang is with the Department of Automation / Beijing National Research Center for Information Science and Technology (BNRist), Tsinghua University, Beijing 100084, China, and also with the Beijing Academy of Artificial Intelligence (BAAI), Beijing 100084, China. E-mail: gaohuang@tsinghua.edu.cn.
}
\thanks{(Corresponding author: Shiji Song.)}

}

%
%

\markboth{IEEE TRANSACTIONS ON PATTERN ANALYSIS AND MACHINE INTELLIGENCE}%
{SHI \textit{et al.}: TEMPORAL-SPATIAL CAUSAL INTERPRETATIONS FOR VISION-BASED REINFORCEMENT LEARNING}
%



\IEEEtitleabstractindextext{%
\begin{abstract}
 Deep reinforcement learning (RL) agents are becoming increasingly proficient in a range of complex control tasks. However, the agent's behavior is usually difficult to interpret due to the introduction of black-box function, making it difficult to acquire the trust of users. Although there have been some interesting interpretation methods for vision-based RL, most of them cannot uncover temporal causal information, raising questions about their reliability. To address this problem, we present a temporal-spatial causal interpretation (TSCI) model to understand the agent's long-term behavior, which is essential for sequential decision-making. TSCI model builds on the formulation of temporal causality, which reflects the temporal causal relations between sequential observations and decisions of RL agent. Then a separate causal discovery network is employed to identify temporal-spatial causal features, which are constrained to satisfy the temporal causality. TSCI model is applicable to recurrent agents and can be used to discover causal features with high efficiency once trained. The empirical results show that TSCI model can produce high-resolution and sharp attention masks to highlight task-relevant temporal-spatial information that constitutes most evidence about how vision-based RL agents make sequential decisions. In addition, we further demonstrate that our method is able to provide valuable causal interpretations for vision-based RL agents from the temporal perspective.
\end{abstract}

}

\maketitle

\IEEEdisplaynontitleabstractindextext

%
\IEEEpeerreviewmaketitle

\IEEEraisesectionheading{\section{Introduction}\label{sec:introduction}}
 \IEEEPARstart{R}{einforcement} learning (RL) is posed as a problem of sequential decision-making in an interactive environment. Recently, deep RL has led to tremendous progress in a variety of challenging domains \cite{mnih2015human, zhang2020designing}. Nevertheless, they are often criticized for being black boxes and their lack of interpretability, which has increasingly been a pressing concern in deep RL. In many real-world scenarios where trust and reliability are critical, however, it is hardly satisfactory to merely pursue state-of-the-art performance. The temporal causal relations between sequential observations and decisions need to be revealed so that the provided insights are trustworthy. The inability to explain and justify their decisions, makes it harder to deploy RL systems in some safety-critical fields such as healthcare \cite{raghu2017continuous, zhang2020designing} and finance \cite{chia2019machines}. Therefore, it is crucial to develop the ability to reason about the behavior of RL agents for acquiring the trust of users.

 Explaining the decision of black-box systems \cite{cao2019interpretable, monfort2019moments, liu2019tabby} is an area of active research, and there are some popular methods developed for generating visual interpretations, such as LIME \cite{ribeiro2016should}, LRP \cite{binder2016layer}, DeepLIFT \cite{shrikumar2017learning}, Grad-CAM \cite{selvaraju2017grad}, Kernel-SHAP \cite{lundberg2017unified} and network dissection \cite{bau2017network, zhou2018interpreting}. However, most of these methods are proposed exclusively for supervised learning, and cannot be directly adapted to sequential decision-making. So far, some existing methods have provided valuable and insightful interpretations for vision-based RL from the spatial perspective, and they typically focus on a better understanding of what information is attended to and why mistakes are made via visualization techniques, such as gradient backpropagation \cite{zahavy2016graying, wang2016dueling}, perturbation injection \cite{greydanus2018visualizing, puri2020explain} and attention mechanism \cite{mott2019towards}. While gradient-based and perturbation-based methods generally interpret single action rather than long-term behavior, attention-augmented methods need to adapt and retrain the agent model to be interpreted. Moreover, they cannot uncover temporal causal information, which is essential for understanding the behavior of RL agents. While a few works \cite{karpathy2015visualizing, bargal2018excitation} have studied temporal interpretations for recurrent neural networks (RNNs), very little work has provided reliable temporal-spatial interpretations for sequential decisions of RL agents.

 In the context of machine learning, RL is distinguished from other learning paradigms due to two unique characteristics. First, single-timestep observation alone is usually not enough to achieve optimal performance for RL environments that are partially observable. Second, an observation is relevant to not only the current decision but also the future decisions. In other words, the behavior of RL agents depends not only on spatial features but also on temporal features, which intuitively can be extracted from consecutive observations of the same environment. Our work is motivated by trying to understand sequential decisions of vision-based RL agents from a temporal-spatial causal perspective. To that end, we draw on the notion of Granger causality \cite{granger1969investigating}, which is based on the intuition that a cause helps predict its effects in the future. Granger causality is an effective approach to reasoning about temporal causal relations between time series involving many features \cite{arnold2007temporal, gong2015discovering}. Prior works \cite{schwab2019granger, schwab2019cxplain} have applied a non-temporal variant of Granger causality for model interpretation in supervised learning, but there is little use in the practice of attempting to discuss causality in RL without introducing time.

 One of our main contributions is that a \emph{Temporal-Spatial Causal Interpretation} (TSCI) model is proposed to understand sequential decisions of vision-based RL agents. TSCI model builds on the concept of temporal causality that characterizes the temporal causal relations between sequential observations and decisions. To identify temporal-spatial causal features, a separate causal discovery network is employed to learn the temporal causality. TSCI model, once trained, can be used to generate causal interpretations about the agent's behavior in little time. Our approach does not depend on a specific way to obtain the agent model and can be applied readily to deep RL agents that use recurrent structures. In particular, it does not require adapting or retraining the original agent model. We conduct comprehensive experiments on Atari 2600 games of the Arcade Learning Environment \cite{bellemare2013arcade}. The empirical results verify the effectiveness of our method, and demonstrate that our method can produce high-resolution and sharp attention masks to highlight task-relevant temporal-spatial information that constitutes most evidence for the agent's behavior. In other words, our method discovers temporal-spatial causal features to interpret how vision-based RL agents make sequential decisions.

 For RL environments that are partially observable, consecutive observations are generally stacked together to enable the learning of temporal representation about high-level semantic features. Another main contribution of this work is to further reveal and understand the role that temporal dependence plays in sequential decision-making. To that end, we perform more experiments to evaluate the impact of temporal dependence on the agent's long-term performance based on counterfactual analysis, and then leverage the proposed TSCI model to further explain the resulting counterfactual phenomenon. In addition, this work provides empirical explanations about why frame stacking is generally necessary even for the agent that uses a recurrent structure from the point of view of temporal dependence. The results demonstrate that the proposed TSCI model can be applied to provide valuable causal interpretations about the agent's behavior from the temporal-spatial perspective.

 The remainder of this paper is organised as follows. In the following two sections, we summarize the related works and give a brief introduction to the preliminaries used in this work. In Section \ref{sec:methodology}, we mainly present a temporal-spatial causal interpretation (TSCI) model for interpreting vision-based RL agents. In Section \ref{sec:experiments}, empirical results are provided to verify the effectiveness of our method. In Section \ref{sec:temporal interpretaions}, the proposed method is applied to further reveal and understand the role that temporal dependence plays in sequential decision-making. In the last section, we draw the conclusion and outline the future work.

\section{Related Work}\label{sec:related work}
\subsection{Interpreting Deep RL Agents}
 There is a substantial body of literature about how to interpret deep RL agents. While the broad objective of RL interpretation is to make RL policies more understandable, each work has its own special purposes, sets of applicable problems, limitations, and challenges \cite{alharin2020reinforcement}. Here we review some popular interpretation methods introduced in previous works.

 \emph{Gradient-based methods} identify input features that are most salient to the trained deep neural network (DNN) by using the gradient to estimate their influence on the output. A feasible approach is to generate Jacobian saliency maps \cite{simonyan2014deep} to visualize which pixels in the state affect the action the most \cite{wang2016dueling}. There are several variants modifying gradient to obtain more meaningful saliency, such as Integrated Gradients \cite{sundararajan2017axiomatic}, Excitation Backpropagation \cite{zhang2018top}, DeepLIFT \cite{shrikumar2017learning} and Grad-CAM \cite{selvaraju2017grad}. Unfortunately, these gradient-based methods depend on the shape in the neighborhood of a few points and are vulnerable to adversarial attacks \cite{ghorbani2019interpretation}. Furthermore, they are unable to provide a valid interpretation from a temporal perspective. To enable video attribution in the temporal dimension, some methods extend Excitation Backpropagation and Grad-CAM to produce temporal maps \cite{bargal2018excitation, stergiou2019saliency}. However, the above problem still remains unsettled, and most of these methods require a well-designed network structure.

 \emph{Perturbation-based methods} measure the variation of a black-box model's output when some of the input information is removed or perturbed \cite{fong2017interpretable, dabkowski2017real}. It is important to choose a perturbation that removes information without introducing any new information. The simplest perturbation approach is to replace part of an input image with a gray square \cite{zeiler2014visualizing} or region \cite{ribeiro2016should}. In order to provide reliable interpretations, some recent works attempt to estimate feature importance by combining the perturbation approach with Granger causal analysis \cite{granger1969investigating}, such as causal explanations (CXPlain) \cite{schwab2019cxplain} and the attentive mixture of experts (AME) \cite{schwab2019granger}. A particular example of the perturbation approach is Shapley values \cite{shapley1953value, lundberg2017unified, ancona2019explaining}, but the exact computation of which is NP-hard. To interpret deep RL agents, there are some works \cite{greydanus2018visualizing, puri2020explain, iyer2018transparency} using perturbation-based saliency maps to understand how an agent learns a policy, although the saliency maps are suggested to be viewed as an exploratory tool rather than an explanatory tool \cite{atrey2020exploratory}. However, such a suggestion seems to be contentious. First, the proposed method does not work for recurrent agents. Second, it emphasizes the consistency of attribution results with the human inspection, which deviates from the target of attribution, i.e., discovering the regions relied upon by models rather than humans.

 \emph{Attention-augmented methods} incorporate various attention mechanisms into the agent model. Learning attention to generate saliency maps for understanding internal decision pattern is one of the most popular methods \cite{wang2020paying} in deep learning community, and there are already a considerable number of works in the direction of interpretable RL. These works are aimed at getting better interpretability while not sacrificing the performance of RL agents. A simple approach is to augment the actor with customized self-attention modules \cite{manchin2019reinforcement, nikulin2019free, sorokin2015deep}, which learn to focus its attention on semantically relevant areas. Another branch of this category implements the key-value structure of attention to learn explainable policies by sequentially querying its view of the environment \cite{mott2019towards, annasamy2019towards, choi2017multi}. However, attention-augmented methods need to adapt and retrain the agent model, making it unable to interpret the agent models that have been trained or whose network structure cannot be changed. Moreover, attention and causality are two different concepts that are associated with interpretability. While attention aims to find the semantic information that is salient to the agent's decision, causality is the relationship between cause and effect. The principle of causality is that everything has a cause. Different from attention-augmented methods, our work draws on the notion of Granger causality, i.e., a cause helps predict its effects in the future, to discover temporal-spatial causal features for reliable interpretations about the RL agent's behavior.

 Besides the above established categories of interpretation methods, structural causal models (SCMs) \cite{madumal2020explainable}, decision trees \cite{bastani2018verifiable} and mimic models \cite{zhang2020atari} have also recently been proposed for deep RL. However, while SCMs learn an action influence model whose causal structure must be given beforehand, the others are designed for specific models or build on human demonstration datasets. Lastly, a major limitation of most existing RL interpretation methods is that they generally interpret single action rather than long-term behavior and cannot uncover temporal causal information. In contrast, this work follows the idea underlying SSINet \cite{shi2020self} to learn an end-to-end interpretation model, but aims to understand the agent's behavior from the temporal-spatial perspective.

\subsection{Causal Analysis of Time Series}
 Another related research field to ours is causal analysis of time series, which aims to find the temporal causal relations from time series. There are some very interesting past works that have explored to reveal the temporal causal information underlying time series data, such as Granger causal analysis \cite{granger1969investigating, gong2015discovering}, graphical Granger methods \cite{eichler2006graphical, arnold2007temporal}, the SIN method \cite{drton2008sinful} and vector autoregression (VAR) \cite{valdes2005estimating, opgen2007learning}. However, these methods are developed exclusively for non-Markovian and low-dimensional time series data, and have not shown the ability to explain the behavior of deep RL agents whose observation spaces are generally high-dimensional, such as images. In contrast, this work builds on the Granger causality to discover temporal-spatial causal features for interpreting the behavior of vision-based RL agents.

\section{Preliminaries}
 Reinforcement learning (RL) is a general class of algorithms that allow an agent to learn how to sequentially make decisions by interacting with an environment $E$. Specifically, the agent takes an action $a_t$ in an observation $o_t$, and receives a scalar reward $r_{t+1}$. Meanwhile, the environment changes its observation to $o_{t+1}$. Then a history $h_t$ is defined as a sequence of observations and actions $``o_0a_0o_1a_1\cdots a_{t-1}o_t"$ that occurred in the past. A state $s_t$ is a summary of all of the information that an agent could possibly have about its current situation, and formally defined as a sufficient statistic for history $h_t$ \cite{wiering2012reinforcement}. In practice, only finite historical observations $``o_{t-m+1}\cdots o_{t-1}o_t"$, denoted by $o_{t-m+1:t}$, are explicitly considered in a state $s_t$, while the others are usually discarded directly or encoded using memory cells, such as recurrent neural networks (RNNs).

 Formally, an RL task can be modelled as a Markov decision process (MDP) with state space $\mathcal{S}$, action space $\mathcal{A}$, initial state distribution $p_0$, transition dynamics $p(s_{t+1}|s_t,a_t)$, and reward function $r_{t+1}=r(s_t, a_t)$. An agent's behavior is defined by a policy $\pi$, which maps a state to a probability distribution over all actions $\pi:\mathcal{S}\rightarrow \mathcal{P}(\mathcal{A})$. The value function of a state $s_t$ under a policy $\pi$, denoted $v_\pi(s_t)$, is the expected sum of discounted future rewards when starting in $s_t$ and following $\pi$ thereafter, i.e., $v_\pi(s_t)=\mathbb{E}_\pi[\sum_{k=0}^{\infty}\gamma^k r_{t+k+1}]$ with a discount factor $\gamma\in[0,1]$. In this work, the agents to be interpreted are obtained with proximal policy optimization (PPO) algorithm \cite{schulman2017proximal}, which uses trust region update to improve a general stochastic policy with gradient descent.

\section{Methodology}\label{sec:methodology}
 In this section, we start with the description of the interpretability problem addressed in this work. Then we present a strict derivation about temporal causal objective, which forms the theoretical foundation of the temporal-spatial causal interpretation (TSCI) model that is proposed afterwards. Finally, a two-stage training procedure is given for training our TSCI model.

\subsection{Problem Setting}
 Consider the setting in which we need to interpret an actor (or agent) model $\pi$ which sequentially takes as input the state $s_t$ to predict the action $a_t$. For the convenience of our formulation and without loss of generality, a state $s_t$ can be reformulated as a set of $p$ temporal features $X^\mathcal{D}_t=\{o^i_{t-m+1:t}, i\in\mathcal{D}\}$ with $\mathcal{D}=\{1,\cdots,p\}$, which represents all available information in the state $s_t$. The temporal feature $o^i_{t-m+1:t}$ is a sequence of past and present values for the $i$-th specific feature $o^i$. Under the above setting, the causality between state and action refers to the causal relation between temporal features $X^\mathcal{D}_{0:T}= \{o^i_{0:T}, i\in\mathcal{D}\}$ and action sequence $a_{0:T}$ over a horizon $T$. Then, our goal is to develop a separate \emph{Temporal-Spatial Causal Interpretation} (TSCI) model $f_{exp}$ that (i) can discover causal features $X^{\mathcal{D}_c}_{0:T}$ ($\mathcal{D}_c \subseteq \mathcal{D}$) from sequential observations, and (ii) is able to interpret the RL agent's behavior from the temporal-spatial perspective.

\subsection{Temporal Causal Objective}\label{subsec:temporal causal objective}
 The core component of our TSCI model is the temporal causal objective that enables us to learn and discover temporal causal features for understanding the long-term behavior of RL agents. The temporal causal objective builds on Granger causality \cite{granger1969investigating}, which has been widely used to find the causal relation from time series. However, the original Granger causal analysis usually assumes a linear model. In this work, we first contribute an adapted version of Granger causality for the RL domain, i.e., \emph{temporal causality} that is independent of the form of agent model.
\begin{definition}[Temporal Causality]\label{def:sequential causality}
 The causality between temporal features $X^{\mathcal{D}_s}_{0:T} (\mathcal{D}_s\subseteq \mathcal{D})$ and action sequence $a_{0:T}$ exists, denoted by $X^{\mathcal{D}_s}_{0:T}\rightarrow a_{0:T}$, if the agent model $\pi$ is able to predict better actions $a_{0:T}$ using all available information $X^\mathcal{D}_{0:T}$ than if the information apart from $X^{\mathcal{D}_s}_{0:T}$ has been used.
\end{definition}

 Given a state $s_t$ (or temporal features $X^\mathcal{D}_t$), we denote $\varepsilon^{\mathcal{D}-\mathcal{D}_s}_t$ as the prediction error without including any information from the temporal features $X^{\mathcal{D}_s}_t$ and $\varepsilon^\mathcal{D}_t$ as the prediction error when considering all available information. To calculate $\varepsilon^{\mathcal{D}-\mathcal{D}_s}_t$ and $\varepsilon^\mathcal{D}_t$, we first compute the predicted actions $a^{\mathcal{D}-\mathcal{D}_s}_t$ and $a^\mathcal{D}_t$ (abbreviated by $a_t$) without and with using $X^{\mathcal{D}_s}_t$, respectively:
 \begin{align}\label{eqn:predicted actions}
  a^{\mathcal{D}-\mathcal{D}_s}_t &= \pi(X^{\mathcal{D}-\mathcal{D}_s}_t), \\
            a_t = a^\mathcal{D}_t &= \pi(X^\mathcal{D}_t).
 \end{align}
 Note that the predicted output is a probability distribution for discrete action spaces. Denote $a^*_t$ and $v^*$ as the optimal action and the optimal state value function respectively, then we can calculate $\varepsilon^{\mathcal{D}-\mathcal{D}_s}_t$ and $\varepsilon^\mathcal{D}_t$:
 \begin{align}
  \label{eqn:action prediction errors}
  \varepsilon^{\mathcal{D}-\mathcal{D}_s}_t &= \mathcal{L}(a^*_t, a^{\mathcal{D}-\mathcal{D}_s}_t) + \alpha \left\|v^*(X^\mathcal{D}_t)-v(X^{\mathcal{D}-\mathcal{D}_s}_t)\right\|_2, \\
  \label{eqn:value prediction errors}
  \varepsilon^\mathcal{D}_t &= \mathcal{L}(a^*_t, a^\mathcal{D}_t) + \alpha \left\|v^*(X^\mathcal{D}_t)-v(X^\mathcal{D}_t)\right\|_2,
 \end{align}
 where $\alpha$ is a small weight coefficient and $\|\cdot\|_2$ denotes the $L_2$-norm. It is worth emphasizing that the consideration of value consistency is optional in both $\varepsilon^{\mathcal{D}-\mathcal{D}_s}_t$ and $\varepsilon^\mathcal{D}_t$, making our method also applicable to the case where the value function is unavailable. The selection of distance measure $\mathcal{L}$ hinges on the type of action space. Here we use Euclidean distance for continuous action space and Wasserstein distance \cite{arjovsky2017wasserstein} for discrete action space. Following the above definition of temporal causality, we define the degree $\Delta\varepsilon^{\mathcal{D}_s}_{0:T}$ to which the temporal features $X^{\mathcal{D}_s}_{0:T}$ causally contributed to the predicted action sequence $a_{0:T}$ as the decrease in the sum of discounted sequential errors
 \begin{align}\label{eqn:discounted errors}
  \Delta\varepsilon^{\mathcal{D}_s}_{0:T} = \sum_{t=0}^{T}\gamma^t\left(\varepsilon^{\mathcal{D}-\mathcal{D}_s}_t - \varepsilon^\mathcal{D}_t\right).
 \end{align}
 Then we have that if $\Delta\varepsilon^{\mathcal{D}_s}_{0:T}>0$, the temporal features $X^{\mathcal{D}_s}_{0:T}$ includes at least one temporal causal feature that causes $a_{0:T}$ according to Definition \ref{def:sequential causality}. This temporal causal relation does not require direct access to the process by which the agent produces its output and thus does not depend on a specific agent model. Formally, our \emph{temporal causal objective} is to discover the temporal causal features $X^{\mathcal{D}_c}_{0:T}$ such that
 \begin{align}\label{eqn:sequential causal objective}
  \Delta\varepsilon^{\mathcal{D}^\prime}_{0:T}>0,~~ \forall~\mathcal{D}^\prime\subseteq\mathcal{D}_c~\text{and}~\mathcal{D}^\prime\neq\varnothing.
 \end{align}

 However, the optimal action $a^*_t$ and the optimal state value $v^*$ in equations (\ref{eqn:action prediction errors}) and (\ref{eqn:value prediction errors}) cannot be directly obtained, hampering the application of temporal causality in RL. To address this issue, we observe that $X^{\mathcal{D}_c}_{0:T}$ can be obtained by leaving out all non-causal temporal features $X^{\mathcal{D}_{nc}}_{0:T}$ such that $\Delta\varepsilon^{\mathcal{D}_{nc}}_{0:T} = \Delta\varepsilon^{\mathcal{D}-\mathcal{D}_c}_{0:T} \leq 0$. Then $a^*_t$ and $v^*$ can be eliminated by applying the following proposition to construct an upper bound for $\Delta\varepsilon^{\mathcal{D}_{nc}}_{0:T}$.
\begin{proposition}\label{pro:surrogate contribution degree}
 For a specific distance measure $\mathcal{L}$, the following bound holds:
 \begin{align}\label{eqn:surrogate contribution degree}
  \begin{split}
   \Delta\varepsilon^{\mathcal{D}_{nc}}_{0:T} &\leq \Delta\hat{\varepsilon}^{\mathcal{D}_{nc}}_{0:T}     \\
                                              &= \sum_{t=0}^{T}\!\gamma^t\!\left(\mathcal{L}(a^{\mathcal{D}_c}_t\!, a^\mathcal{D}_t)+\alpha\!\left\|v(X^{\mathcal{D}_c}_t)-v(X^\mathcal{D}_t)\right\|_2\right)\!.
  \end{split}
 \end{align}
\end{proposition}
\begin{proof}
 In Euclidean geometry and some other geometries, the distance measure $\mathcal{L}$ satisfies the triangle inequality theorem, which states that for any triangle, the sum of the lengths of any two sides must be greater or equal to the length of the remaining side. Therefore, we have
 \begin{align}
  \Delta\varepsilon^{\mathcal{D}_{nc}}_{0:T}
  & = \sum_{t=0}^{T}\gamma^t\left(\varepsilon^{\mathcal{D}-\mathcal{D}_{nc}}_t - \varepsilon^\mathcal{D}_t\right)    \\
  & = \sum_{t=0}^{T}\gamma^t\left(\varepsilon^{\mathcal{D}_c}_t - \varepsilon^\mathcal{D}_t\right)    \\
  & = \sum_{t=0}^{T}\!\gamma^t\!\bigg(\mathcal{L}(a^*_t, a^{\mathcal{D}_c}_t) + \alpha \big\|v^*(X^\mathcal{D}_t) - v(X^{\mathcal{D}_c}_t)\big\|_2\! \\ \nonumber
  &~~~~~~~~~~~~~~~ - \mathcal{L}(a^*_t, a^\mathcal{D}_t) - \alpha \!\left\|v^*\!(X^\mathcal{D}_t) - v(X^\mathcal{D}_t)\right\|_2\!\bigg)\!   \\
  & \leq \sum_{t=0}^{T}\gamma^t\!\left(\!\mathcal{L}(a^{\mathcal{D}_c}_t, a^\mathcal{D}_t)+\alpha \big\|v(X^{\mathcal{D}_c}_t)-v(X^\mathcal{D}_t)\big\|_2\!\right)\! \\
  & = \Delta\hat{\varepsilon}^{\mathcal{D}_{nc}}_{0:T}.
 \end{align}
\end{proof}

 Proposition \ref{pro:surrogate contribution degree} shows the possibility to formulate the temporal causality without the need to provide $a^*_t$ and $v^*$. More concretely, the inequality $\Delta\varepsilon^{\mathcal{D}_{nc}}_{0:T} \leq 0$ can be guaranteed by limiting the upper bound to no greater than zero, i.e., $\Delta\hat{\varepsilon}^{\mathcal{D}_{nc}}_{0:T} \leq 0$. On the other hand, we have $\Delta\hat{\varepsilon}^{\mathcal{D}_{nc}}_{0:T} \geq 0$ according to the definition in (\ref{eqn:surrogate contribution degree}). Based on the above observations, an easy-to-implement variant of temporal causal objective (\ref{eqn:sequential causal objective}) is to leave out all non-causal temporal features $X^{\mathcal{D}_{nc}}_{0:T}$ that satisfies
 \begin{align}\label{eqn:variant of sequential causal objective}
  \Delta\hat{\varepsilon}^{\mathcal{D}_{nc}}_{0:T} = \Delta\hat{\varepsilon}^{\mathcal{D}-\mathcal{D}_c}_{0:T} = 0.
 \end{align}

 The above causality analysis method is applicable not only to high-dimensional image data but also to low-dimensional vector data. However, one of the main purposes of this work is to understand the RL agent's sequential decision-making from the temporal perspective, hence we only focus on vision-based RL environments that are partially observable and thus show obvious temporal dependence between consecutive observations of a state. In contrast, some RL environments that use vector data are usually fully observable, such as MuJoCo \cite{todorov2012mujoco}. For vision-based RL environments that have high-dimensional state space, another challenge of temporal causality is that it is difficult to separate temporal features $o^i_{0:T} (i\in \mathcal{D})$ from each other. In non-temporal scenarios, a sensible method is to group non-overlapping regions of adjacent pixels into super-pixels \cite{schwab2019cxplain, lundberg2017unified}. However, the same semantic feature may be in different locations of images at every timestep, hence the ``super-pixels'' method is not feasible in our temporal setting. To tackle this challenge, we use DNNs to build our TSCI model for vision-based RL in the next section. An advantage of deep models is that they can extract high-level feature representations from high-dimensional data \cite{lecun2015deep}, and thus remove the need to perform manual feature engineering.

 \begin{figure*}[t]
	\setlength{\abovecaptionskip}{-0.01cm}
	\setlength{\belowcaptionskip}{-0.20cm}
	\begin{center}
		\includegraphics[width=0.95\linewidth]{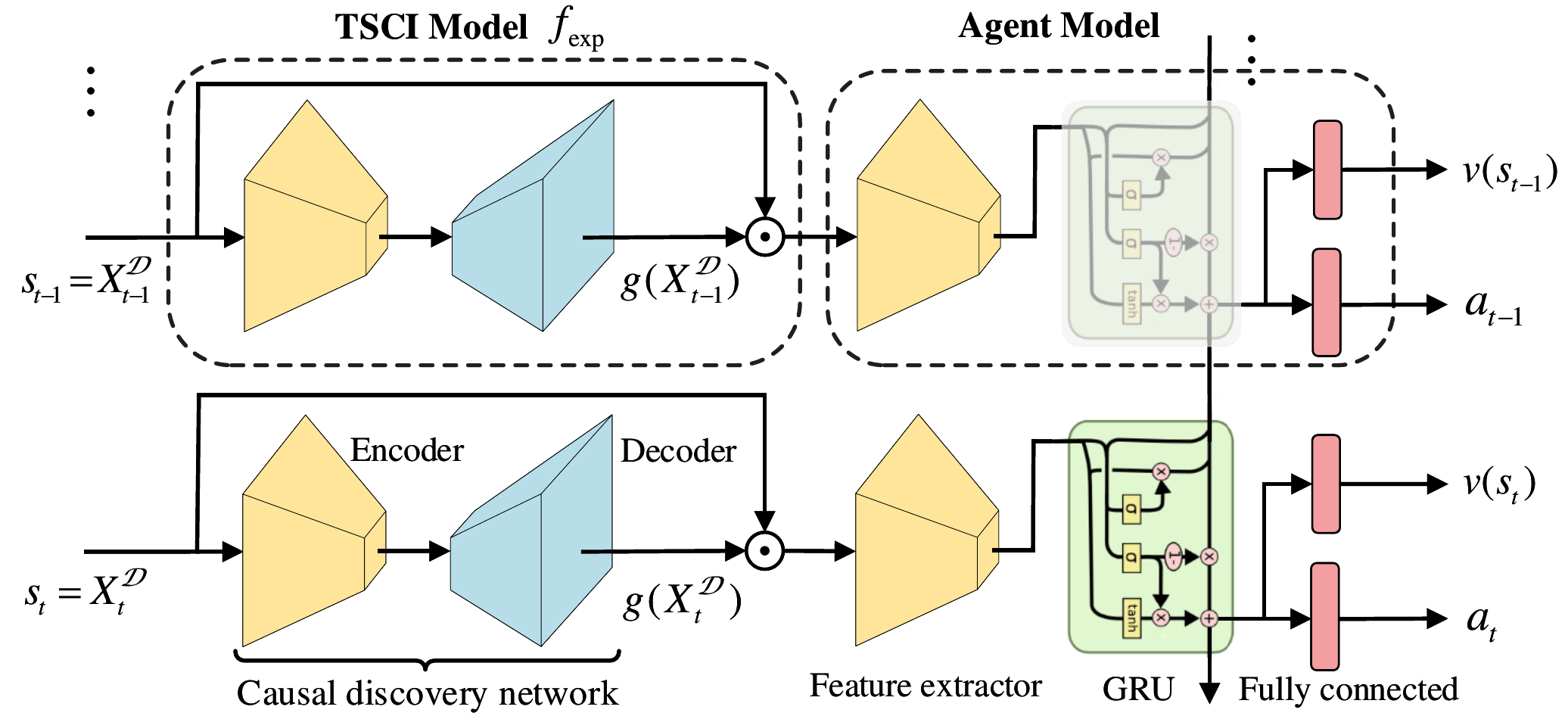}
	\end{center}
	\caption{Architecture diagram of our TSCI model and the agent model to be interpreted. The agent model consists of a feature extractor, a Gated Recurrent Unit (GRU) \cite{cho2014properties} and two fully-connected layers. TSCI model mainly involves an encoder-decoder structure with the encoder shared from the feature extractor. The agent model and the encoder are fixed during training.}
	\label{fig:cdnet}
\end{figure*}
\subsection{Temporal-Spatial Causal Interpretation Model}
 Here we first explain the temporal-spatial causality considered in this work, then we present a trainable TSCI model that learns to generate temporal-spatial causal interpretations about the sequential decisions of vision-based RL agent in an end-to-end manner.

\textbf{Temporal-Spatial Causality.}
 In this work, temporal-spatial causality is based on the intuition that a cause helps predict its effects in the future, and it is supposed to suggest informative explanations that accurately represent the intrinsic reasons for the agent's decision-making from both the spatial and temporal dimensions. More concretely, spatial causality aims to uncover task-relevant semantic features that affect the agent's decision-making, i.e., what information is important and where to look, while temporal causality focuses on revealing the underlying temporal dependence between temporally adjacent observations and how the feature importance varies as time goes on.

\textbf{Training Objective.}
 Based on temporal-spatial causality, TSCI model aims to predict which parts of the input state are considered causal for the agent's behavior. To that end, TSCI model learns a mask generation function $g(X^\mathcal{D}_t)$, taking value between 0 and 1, to discover causal features
 \begin{align}\label{eqn:sequential causal features}
  X^{\mathcal{D}_c}_t=f_{exp}(X^\mathcal{D}_t)= X^\mathcal{D}_t\odot g(X^\mathcal{D}_t),
 \end{align}
 where $\odot$ denotes the element-wise multiplication. The generated masks map sequential input pixels to saliency scores, which reflect the relative degree to which the causal features at every timestep causally contribute to the action. A reasonable choice of training objective is to adopt the variant of temporal causal objective (\ref{eqn:variant of sequential causal objective}), which is better for optimization than the vanilla form (\ref{eqn:sequential causal objective}). Substituting equation (\ref{eqn:sequential causal features}) into (\ref{eqn:variant of sequential causal objective}), we have that $X^{\mathcal{D}_{nc}}_{0:T}$ is the maximal set satisfying
 \begin{align}\label{eqn:non-causal feature series}
  \begin{split}
   \Delta\hat{\varepsilon}^{\mathcal{D}_{nc}}_{0:T} = \sum_{t=0}^{T}\gamma^t \Big(&\mathcal{L}\big(\pi(f_{exp}(X^\mathcal{D}_t)), \pi(X^\mathcal{D}_t)\big) + \\
   &\alpha\big\|v(f_{exp}(X^\mathcal{D}_t))-v(X^\mathcal{D}_t)\big\|_2\Big) = 0,
  \end{split}
 \end{align}
 which requires the agent’s behavior to be consistent with the original after the states are overlaid with the attentions generated by $f_{exp}$. Taking the above conditions into consideration, the objective function is defined to minimize the upper bound of contribution degree $\Delta\hat{\varepsilon}^{\mathcal{D}_{nc}}_{0:T}$ added with a sparse regularization term
 \begin{align}\label{eqn:objective function}
  \mathcal{L}_{TSCI}(g) = \Delta\hat{\varepsilon}^{\mathcal{D}_{nc}}_{0:T} - \beta\sum_{t=0}^{T}\left\|1-g(X^\mathcal{D}_t)\right\|_1,
 \end{align}
 where $\|\cdot\|_1$ denotes the $L_1$-norm, and $\beta$ is a coefficient controlling the sparseness of the mask. In fact, the sparse regularization term requires $f_{exp}$ to attend to as little information as possible, enabling easy understanding of decision-making for humans. In total, this training objective leads to adversarial masks and is composed of two terms. The first term ensures that the change of prediction errors, after the non-causal temporal features are removed, is close to zero. The second term encourages that the masked non-causal feature region is large and thus pushes for better compliance with the temporal causal objective.

\textbf{Causal Discovery Network.}
 As mentioned above, the main purpose of the causal discovery network is to learn the mask generation function $g(X^\mathcal{D}_t)$, which produces an attention mask to highlight the task-relevant information for making decision. To that end, the causal discovery network must learn which parts of the state are considered important by the agent. In the field of computer vision, learning the mask is a dense prediction task, which arises in many vision problems, such as semantic segmentation \cite{ronneberger2015u, lin2017refinenet} and scene depth estimation \cite{mayer2016large}. In order to make the masks sharp and precise, we adapt a U-Net \cite{ronneberger2015u} architecture to build the causal discovery network for TSCI model, as depicted in Figure \ref{fig:cdnet}. As discussed in previous section, the input state is composed of temporally extended observations, hence it is no longer necessary to design a recurrent structure for the causal discovery network, the detailed architecture of which is given in Appendix A.1 of the supplementary material. In particular, instead of learning $g(X^\mathcal{D}_t)$ from scratch, we directly reuse the feature extractor of the agent model as the encoder. This greatly reduces the risk of overfitting and ensures that the generated masks are semantically consistent with the agent model. Therefore, we just need to optimize the decoder with the temporal causal objective.

\subsection{Training Procedure}
 We train the causal discovery network to directly minimize the objective function (\ref{eqn:objective function}) with supervised learning. The weights of the encoder (yellow block on the left in Figure \ref{fig:cdnet}) are kept fixed during the training.

 Our training procedure includes two stages. In the first stage, to capture the temporal causal relations between states and actions at different timesteps, temporal data is collected for training. More concretely, we first use the agent model to collect $M$ episodes with a fixed horizon $T$. Each episode is divided into a state sequence, an action sequence and a state value sequence. While the state sequences are regarded as the input data, the action and state value sequences together form the label. Once the episode dataset is built, a standard supervised learning procedure is then applied to train the causal discovery network in the second stage. The pseudo-code of training is summarized in Algorithm \ref{alg:TSCI}. In fact, an alternative approach to training our causal discovery network is to use single-timestep state-action pairs and non-temporal objective functions. In our experiments, we will make a fair comparison between our method and similar methods that use non-temporal objective functions.
 \setlength{\algomargin}{1.5em}
 \begin{algorithm}[h]
    \caption{The training procedure of TSCI model}
    \label{alg:TSCI}
    The agent model or actor $\pi$ to be interpreted and corresponding critic $v$\;
    Initialize a causal discovery network $g$. The encoder is initialized with the feature extractor of $\pi$, while the decoder is initialized with random weights\;
    Use $\pi$ to collect $M$ episodes with a fixed horizon $T$, and build the dataset for training\;
    \For{Epoch = 1, K}
    {
      \For{Iteration = 1, $M$ mod $N$}
      {
        Sample $N$ episodes\;
        Calculate the objective function defiend by Equation \eqref{eqn:objective function}\;
        Update the weights of decoder while the weights of encoder are fixed\;
      }
    }
 \end{algorithm}
 \begin{figure*}[t]
  \setlength{\abovecaptionskip}{-0.01cm}
  \setlength{\belowcaptionskip}{-0.10cm}
  \begin{center}
   \includegraphics[width=1.0\linewidth]{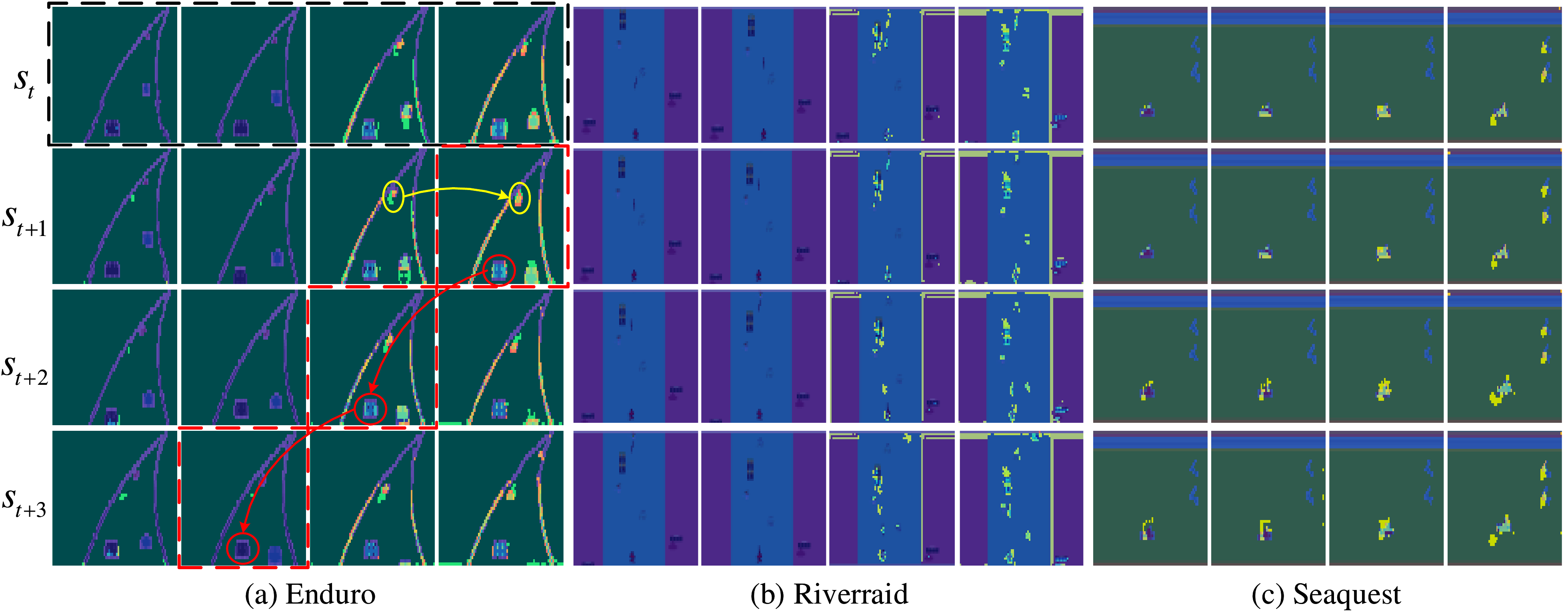}
  \end{center}
  \caption{Visualization of temporal-spatial causal features discovered by our method. The last four frames are explicitly considered for each state as shown by the dashed black rectangle (time goes from left to right). The frames on the diagonal are identical as shown by the dashed red rectangles.}
  \label{fig:temporal causal features}
 \end{figure*}

\section{Validity of Our Method}\label{sec:experiments}
 Before we apply the proposed method to render temporal-spatial causal interpretations for vision-based RL agents, we first verify the effectiveness of our method through performance evaluation and comparative evaluation in this section. Subsequently, in Section \ref{sec:temporal interpretaions}, the proposed TSCI model is applied to further reveal and understand the role that temporal dependence plays in sequential decision-making.

\subsection{Experiment Setup}
 We conduct extensive experiments on Atari 2600 games of the Arcade Learning Environment \cite{bellemare2013arcade}, which is a widely used benchmark in the field of RL interpretability. The agent models to be interpreted are pretrained with an actor-critic setup and the standard PPO training procedure. Then we apply the agent model to generate $10^4$ episodes with a fixed horizon 64. Finally, the decoder of causal discovery network is trained using the collected episode data. More details regarding task setup and training hyperparameters are provided in Appendix A.2 of the supplementary material. To enable fair and meaningful evaluations, we mainly select vision-based tasks for three reasons. First, we are better able to manipulate the state on vision-based tasks for some specific purposes. Second, most of the existing methods that we wish to compare to are developed exclusively for vision-based tasks. Third, these vision-based tasks are generally partially observable, making it convenient to verify the underlying temporal dependence of sequential decision-making. Nevertheless, we note that the temporal causality introduced in Section \ref{subsec:temporal causal objective} is compatible with any deep RL algorithm and task.

 \begin{figure*}[t]
  \setlength{\abovecaptionskip}{-0.01cm}
  \setlength{\belowcaptionskip}{-0.05cm}
  \begin{center}
   \includegraphics[width=0.95\linewidth]{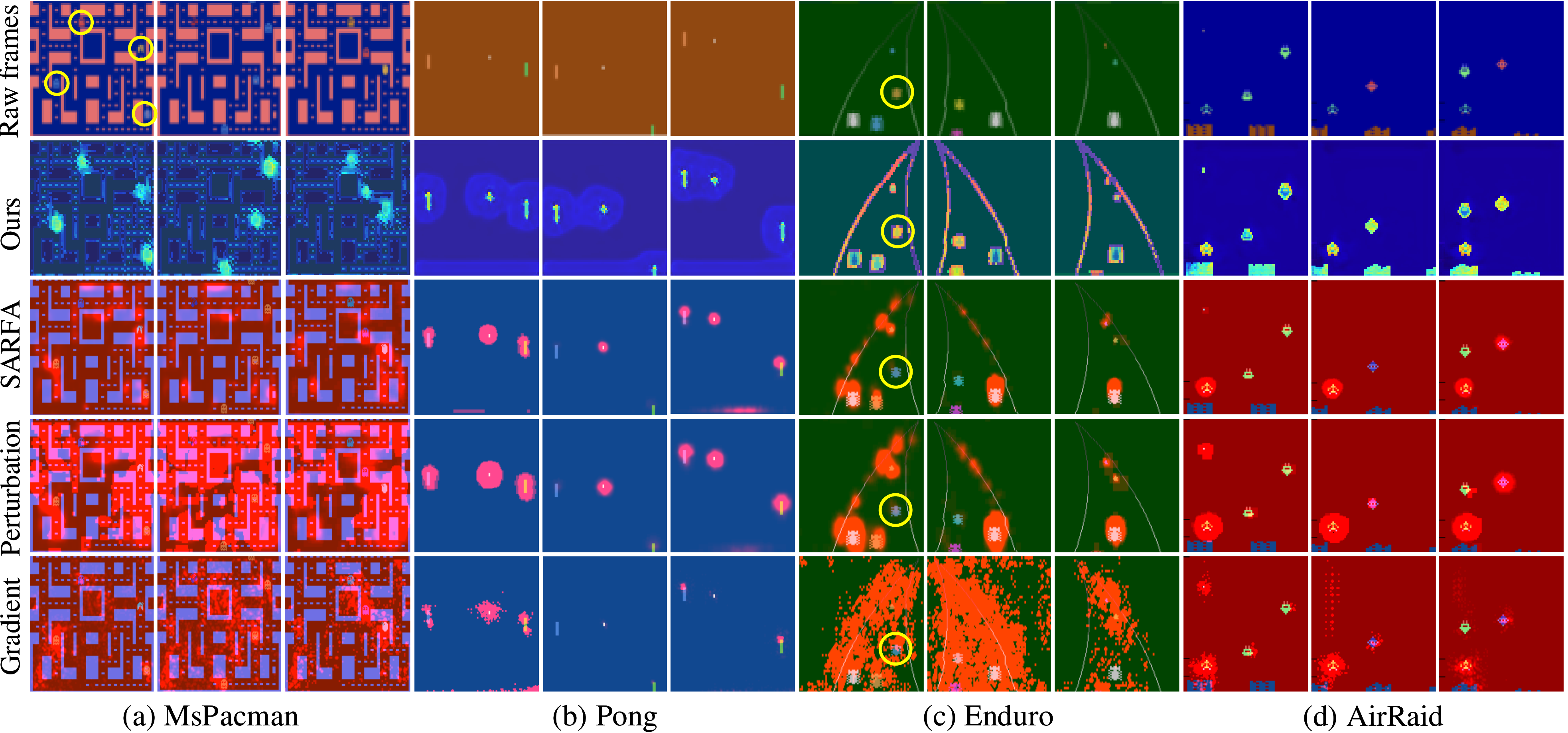}
  \end{center}
  \caption{Comparing saliency maps generated by different methods including our method, SARFA \cite{puri2020explain}, Gaussian perturbation method \cite{greydanus2018visualizing} and gradient-based method \cite{wang2016dueling}.}
  \label{fig:saliency comparisons}
 \end{figure*}
\subsection{Evaluations}
 The main goal of our evaluations is to demonstrate the effectiveness of our proposed TSCI model by answering the following three questions:

 \textbf{\emph{Question 1: Is our method capable of discovering temporal-spatial causal features for a better interpretation of the RL agent's sequential decisions?}}~
 Figure \ref{fig:temporal causal features} visualizes the temporal-spatial causal features discovered by our method and reveals how they changes over time. The most dominant pattern is that the agent focuses attention selectively on only small regions which are strongly task-relevant at each timestep, while other regions are very ``blurry'' and can be ignored. In other words, the agent learns what information is important for making decisions and where to look at each timestep. In addition, the agent's sequential decisions can be understood from at least two aspects. First, the agent's decision is causally attributed to not only the features of current timestep but also that of past timesteps. For example, as the yellow ellipses show, both the driving cars in the last two frames are found to be remarkably causal for the decision. In fact, multiple observations of the same object can provide information such as motion direction, velocity or acceleration. Second, the same causal features contribute to the actions in varying degrees at different timesteps. As illustrated by the red circles, the cars observed at the past timesteps become decreasingly important for the current decision as time goes on. More results on other tasks are provided in Appendix B of the supplementary material.

 In particular, it is worth noting that the results in Figure \ref{fig:temporal causal features} only represent the relative importance of discovered features at different timesteps. Therefore, the first two frames may still be important for the agent's decision-making, though they are visually less salient than the last two frames. In fact, the degree of feature saliency is related to the choice of regularization coefficient $\beta$ in Equation \eqref{eqn:objective function}. The ablation analysis of $\beta$ is given in Appendix A.2 of the supplementary material.

 \textbf{\emph{Question 2: How does TSCI compare to existing RL interpretation methods in terms of the quality of discovered features?}}~
 We compare TSCI against several popular RL interpretation methods including specific and relevant feature attribution (SARFA) \cite{puri2020explain}, Gaussian perturbation method \cite{greydanus2018visualizing} and gradient-based method \cite{wang2016dueling}. Here we do not consider attention-augmented methods for comparisons, since they generally require adapting and retraining the agent model to be interpreted. Figure \ref{fig:saliency comparisons} shows the saliency maps generated by different methods. It can be seen that our method produces higher-resolution and sharper saliency maps than the others. As illustrated by the yellow circles on MsPacman task, our method is able to locate precisely all causal features. In contrast, other methods either highlight lots of non-causal features or omit some causal features, as shown by the yellow circles on Enduro task.

 To quantitatively evaluate the quality of features discovered by different methods, we further compare the average return of several policies that can access to only the pixels of particular features obtained by different methods during the training. Table \ref{tab:comparison return} summarizes the results of four methods, and all results are averaged across five random training runs. It can be seen that the policy still achieves good performance when accessing to only the pixels of the features discovered by our method. In contrast, we can observe obvious performance degradation when using the other methods to generate the features.
 \begin{table}[h]
  \centering
  \fontsize{9.0}{10}\selectfont
  \setlength{\tabcolsep}{1.5mm}{}
  \caption{The performance comparison of the features discovered by different methods.}
  \label{tab:comparison return}
  \begin{threeparttable}
   \begin{tabular}{ccccc}
    \toprule
     Tasks        & Ours          & SARFA \cite{puri2020explain}  & Perturbation \cite{greydanus2018visualizing} & Gradient \cite{wang2016dueling} \cr
    \midrule
    Enduro        & \textbf{2903} & 2369   & 1741         & 819      \cr
    Seaquest      & \textbf{2517} & 2085   & 1830         & 836      \cr
    \bottomrule
   \end{tabular}
  \end{threeparttable}
 \end{table}

 \begin{figure*}[t]
  \setlength{\belowcaptionskip}{-0.20cm}
  \centering
  \begin{minipage}[t]{0.32\textwidth}
   \centering
   \includegraphics[width=0.88\linewidth]{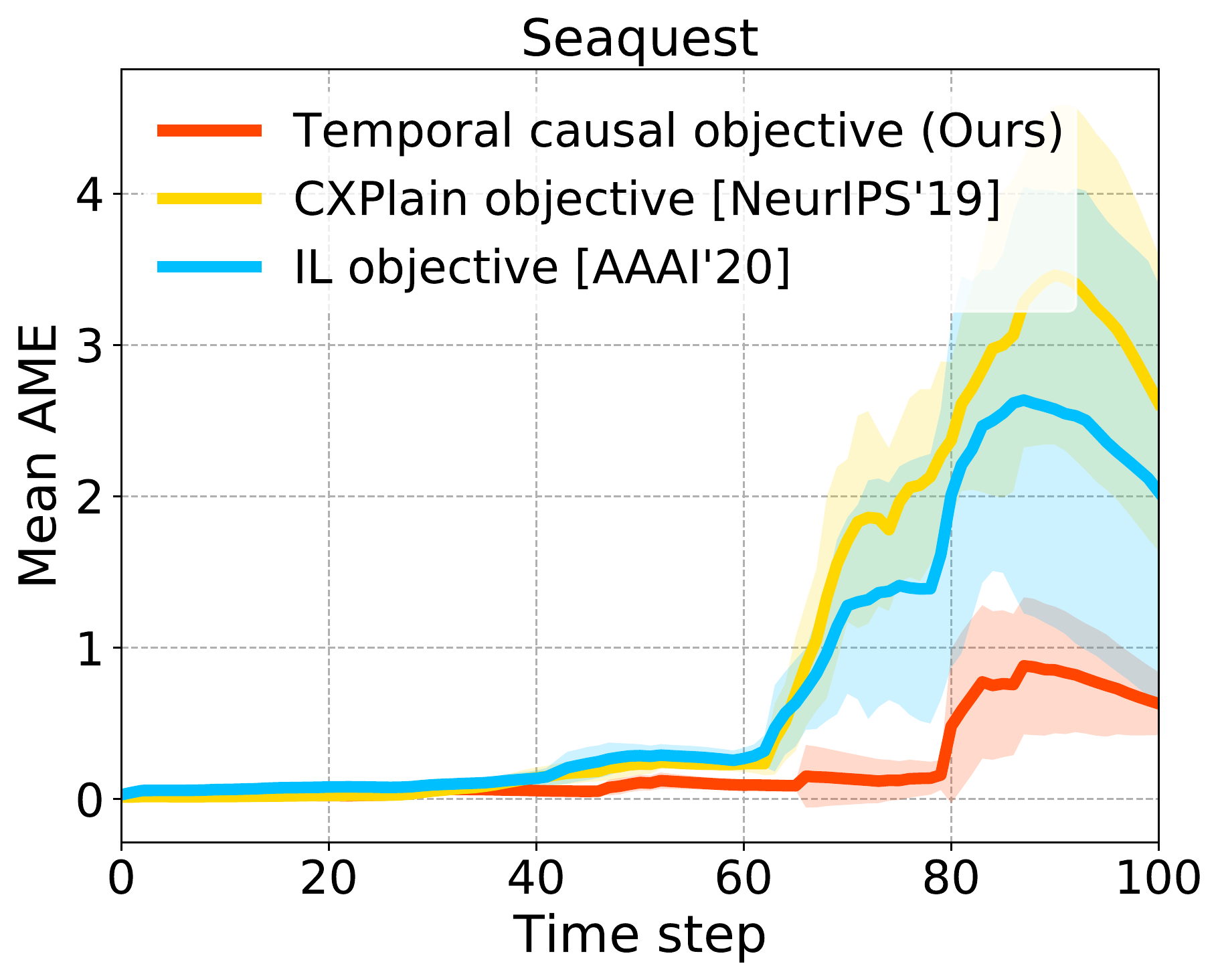}
  \end{minipage}
  \begin{minipage}[t]{0.32\textwidth}
   \centering
   \includegraphics[width=0.95\linewidth]{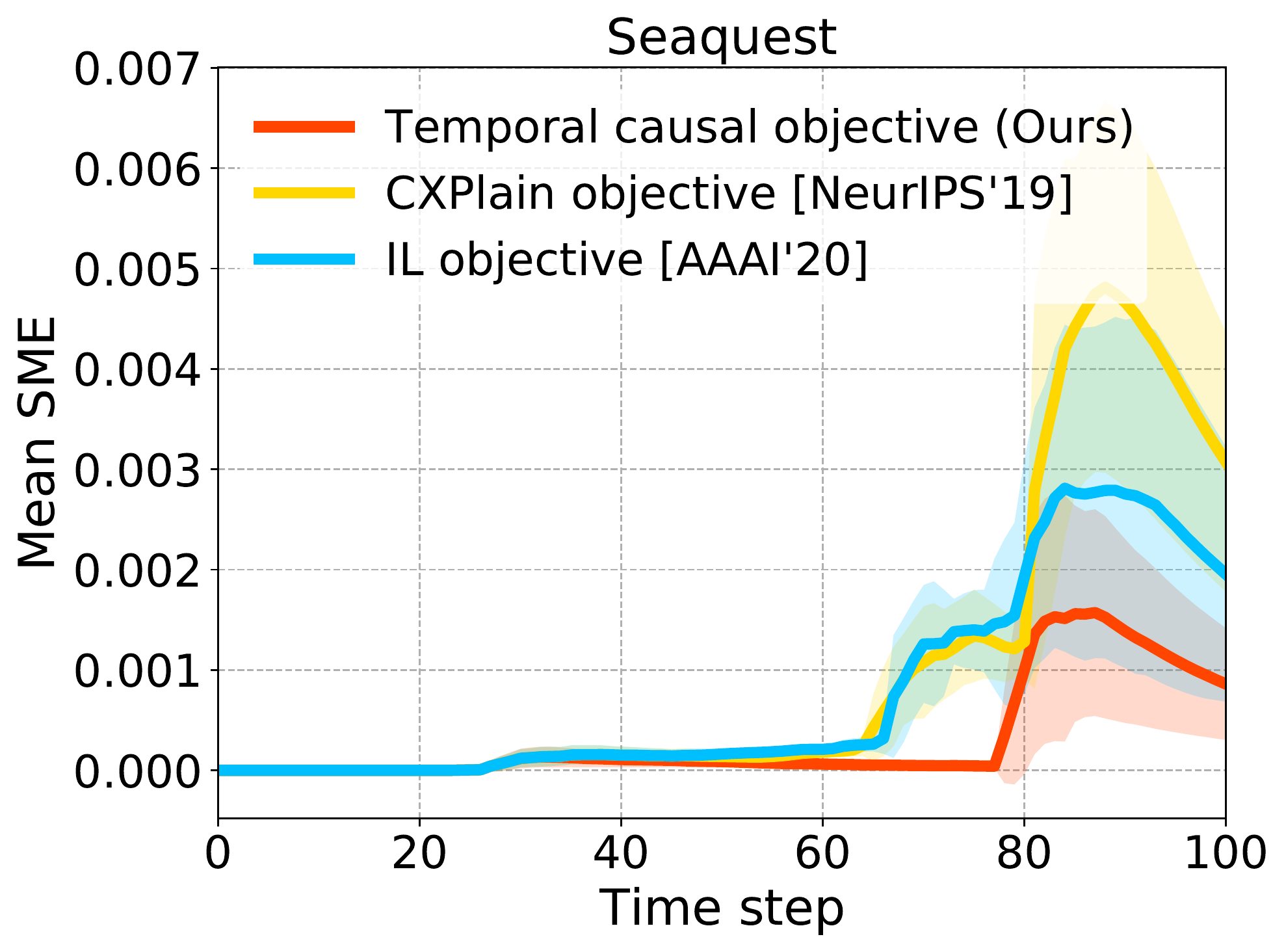}
  \end{minipage}
  \begin{minipage}[t]{0.32\textwidth}
   \centering
   \includegraphics[width=0.95\linewidth]{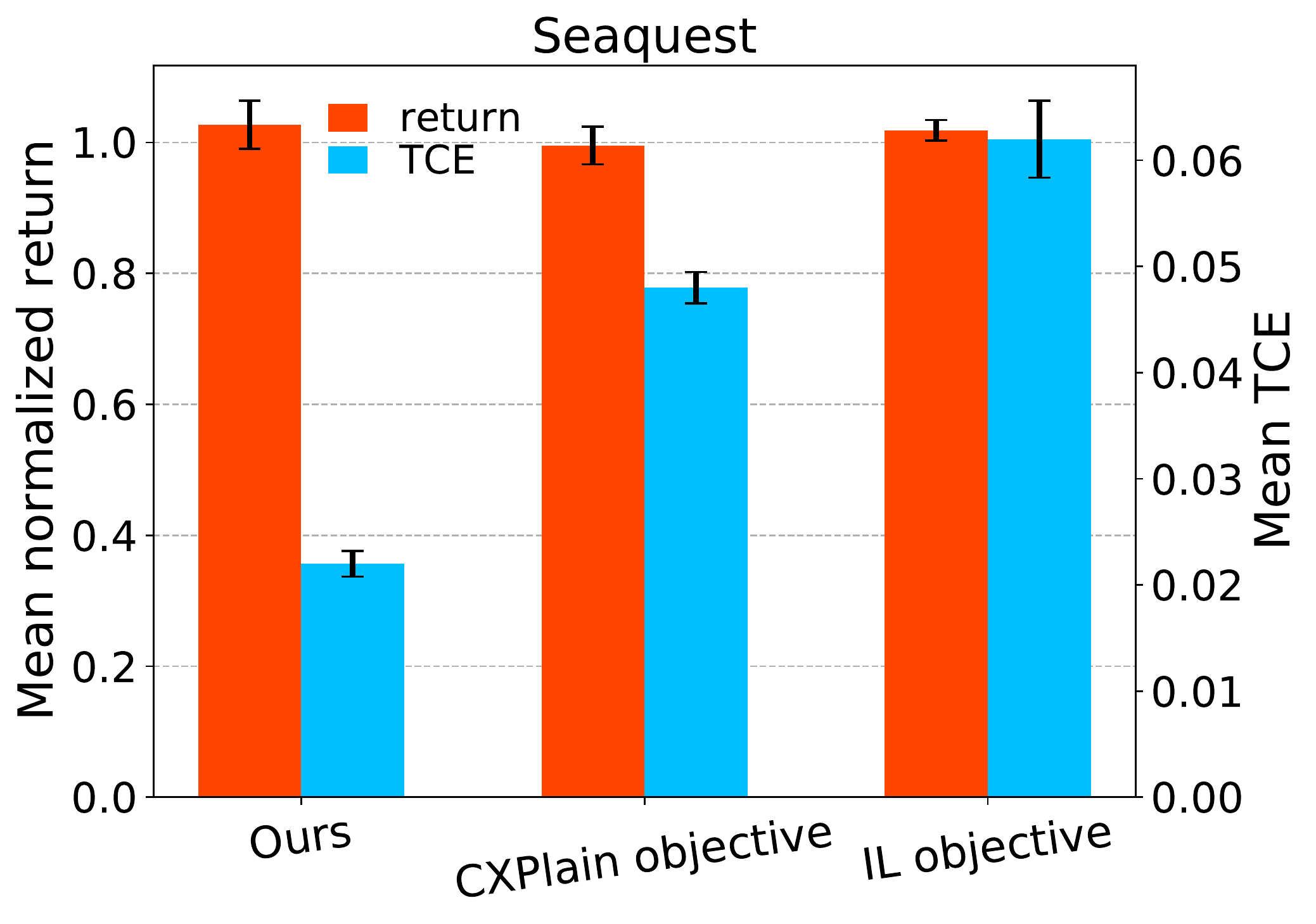}
  \end{minipage}
  \begin{minipage}[t]{0.32\textwidth}
   \centering
   \includegraphics[width=0.92\linewidth]{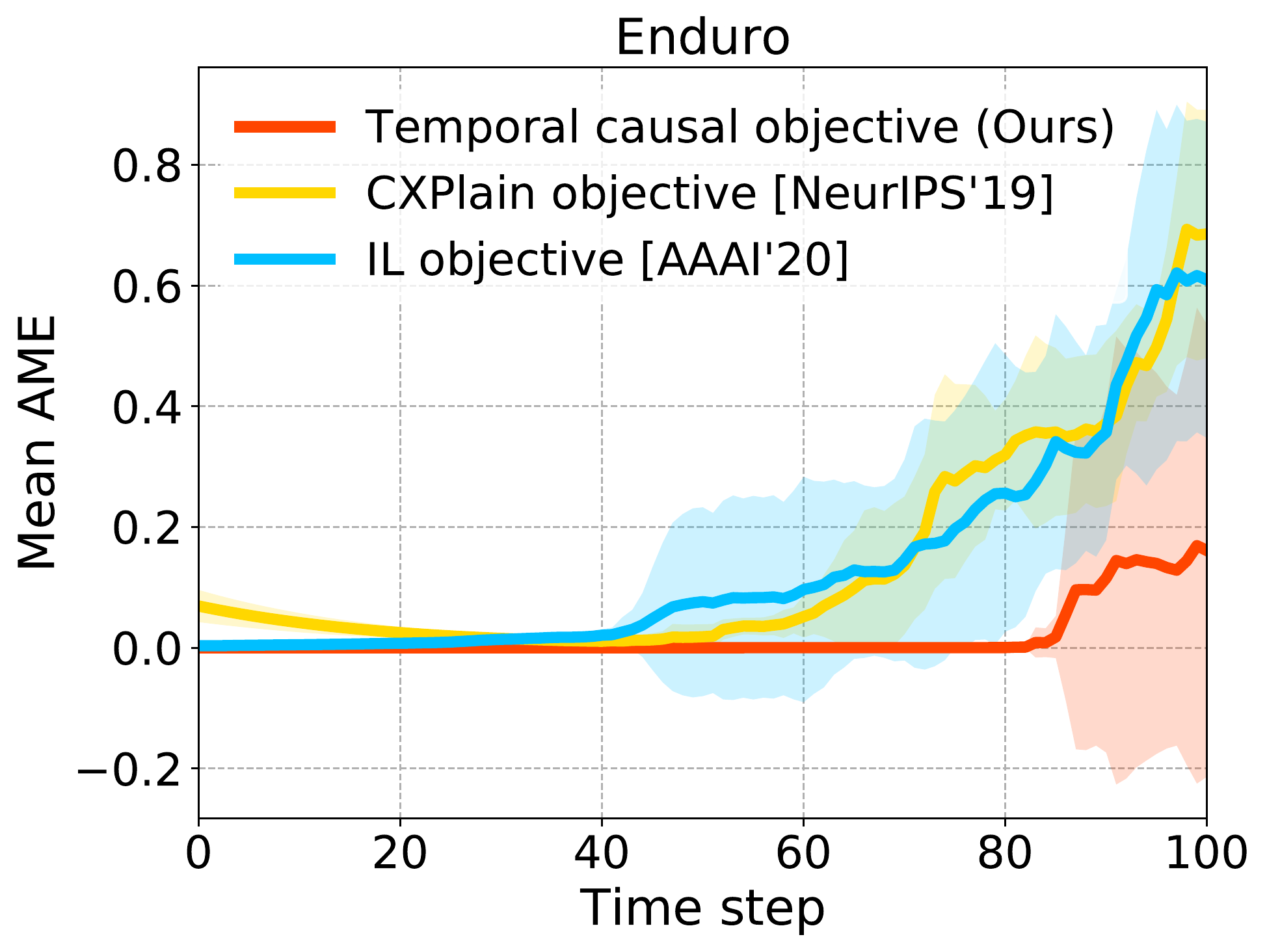}
  \end{minipage}
  \begin{minipage}[t]{0.32\textwidth}
   \centering
   \includegraphics[width=0.95\linewidth]{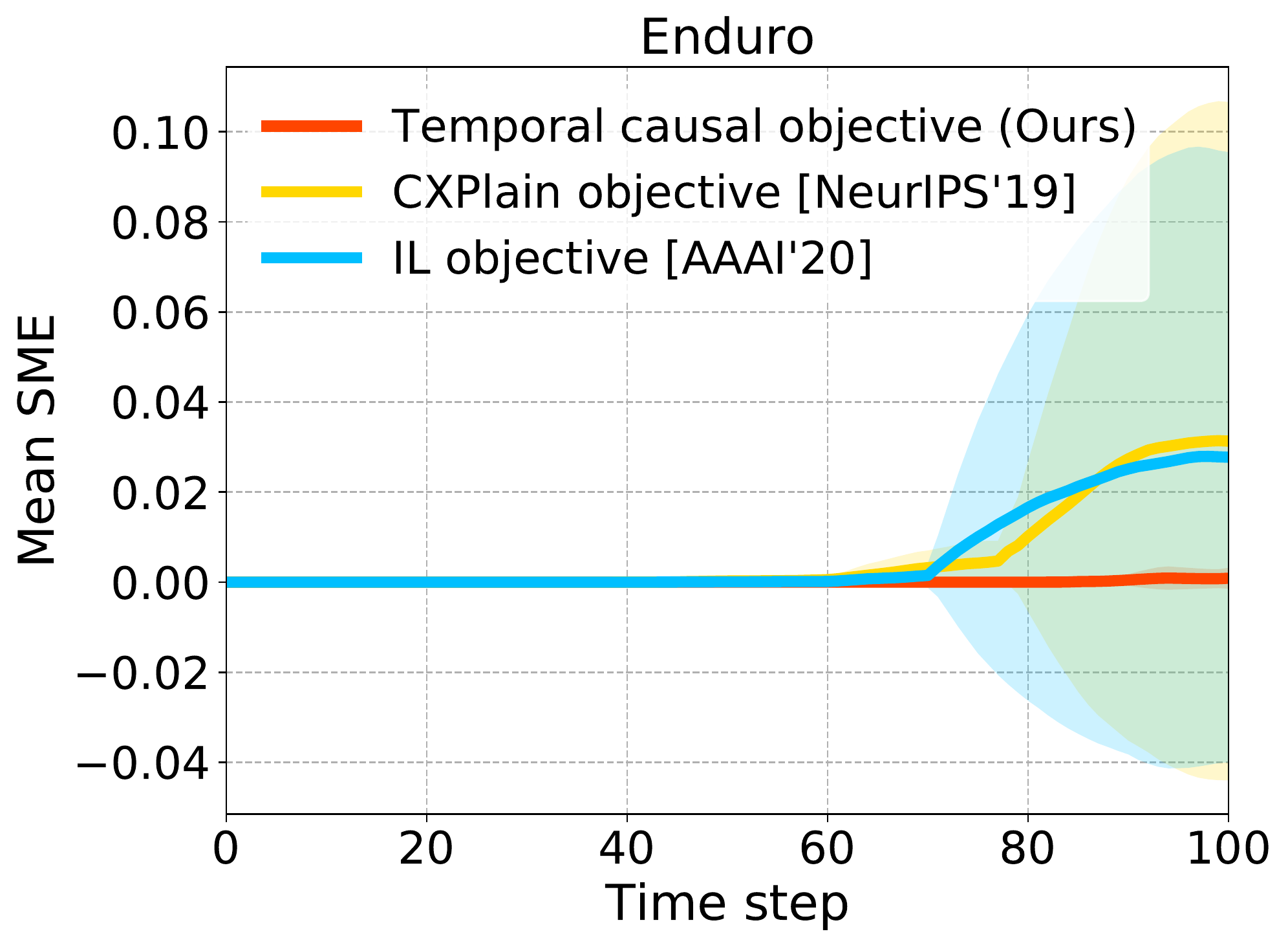}
  \end{minipage}
  \begin{minipage}[t]{0.32\textwidth}
   \centering
   \includegraphics[width=0.95\linewidth]{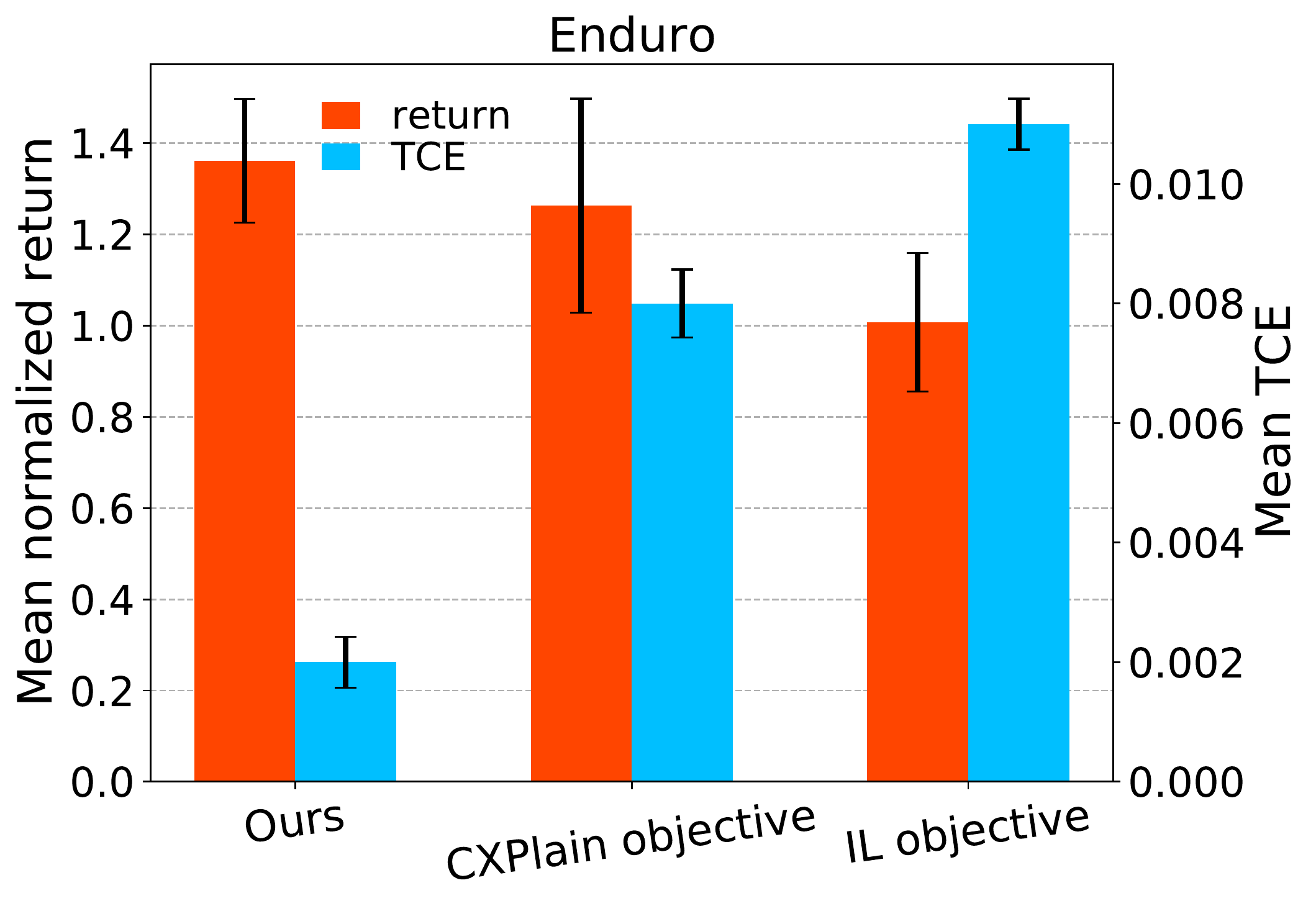}
  \end{minipage}
  \caption{Reliability evaluation metrics. Comparing against CXPlain objective \cite{schwab2019cxplain} and imitation learning (IL) objective \cite{zhang2020atari}. All results are averaged across seven evaluation episodes.}
  \label{fig:evaluation metrics}
 \end{figure*}
 \textbf{\emph{Question 3: Does our temporal causal objective enable reliable causal interpretations about the RL agent's long-term behavior?}}~
 As presented in Section \ref{sec:methodology}, our TSCI model mainly relies on the temporal causal objective to interpret the RL agent's long-term behavior, hence we consider two recently popular objectives for comparisons against our temporal causal objective. 1) CXPlain objective \cite{schwab2019cxplain}, which also builds on the definition of Granger causality \cite{granger1969investigating} used by our TSCI but in a non-temporal form. In our implementation, we apply Proposition \ref{pro:surrogate contribution degree} to avoid the use of the optimal actions in RL. 2) Imitation learning (IL) objective \cite{zhang2020atari, pomerleau1991efficient}, which treats the behavior matching task as a multi-class pattern problem with a standard log-likelihood objective. In particular, similar to temporal causal objective, we add the single-timestep value prediction error in CXPlain and imitation objectives for the fairness of comparison. It is worth emphasizing that both CXPlain and imitation objectives use independent state-action pairs for training, while our temporal causal objective treats an episode (or state-action sequence) as a sample. Therefore, different from the others, our temporal causal objective minimizes the cumulative discounted errors rather than single-timestep prediction error. Totally, all of these methods use the same network architecture except that the objective functions used for training are different.

 Now we consider the metrics used for comparisons. In the context of sequential decision-making, we are usually more concerned with the agent's long-term return than one-step reward. Therefore, a reasonable way to evaluate the reliability is to measure the degree of long-term behavior matching rather than one-step action matching. To that end, we use four evaluation metrics to measure the reliability of the generated causal interpretations, namely temporal causality error (TCE) $e_{tc}$, normalized return $\overline{R}$, action matching error (AME) $e_a(t)$ and state matching error (SME) $e_s(t)$. Specifically, suppose the trajectories $\{s_0,\pi(s_0),r_1,s_1,\pi(s_1)\cdots\}$ and $\{s_0,\pi(f_{exp}(s_0)),\hat{r}_1,\hat{s}_1,\pi(f_{exp}(\hat{s}_1))\cdots\}$ are generated by the agent model $\pi$ taking as input all available information and only the causal features discovered by $f_{exp}$ respectively, then the TCE $e_{tc}$ is calculated using equation \eqref{eqn:non-causal feature series} in an undiscounted form while the other evaluation metrics are calculated as follows:
 \begin{align}\label{eqn:evaluation metrics}
    \overline{R} &= \sum\nolimits_t\hat{r}_t\Big/\sum\nolimits_t r_t, \\
          e_a(t) &= KL\big(\pi(f_{exp}(\hat{s}_t)), \pi(s_t)\big), \\
          e_s(t) &= \big\|(v(f_{exp}(\hat{s}_t)) - v(s_t)\big\|^2_2,
 \end{align}
 where $KL(\cdot)$ denotes the Kullback-Leibler divergence. In fact, similar metrics are suggested in prior work \cite{zhang2020atari}. In the above setting, two episodes are obtained by performing different rollouts from the same initial state, indicating that the SME $e_s(t)$ and AME $e_a(t)$ depend on the whole trajectory before time-step $t$ rather than only the current state-action pair. Consequently, these metrics are able to measure the consistency of long-term behaviors between two trajectories on deterministic RL environments.

 Figure \ref{fig:evaluation metrics} shows the comparison results of all three methods. More results on other tasks are provided in Appendix B of the supplementary material. It can be seen that although all three methods have no performance loss in terms of mean normalized return, our TSCI model that uses temporal causal objective has smaller behavior matching errors and temporal causality error than the others. Hence our temporal causal objective enables reliable interpretations about the RL agent's long-term behavior from the temporal causality perspective.

\section{Temporal Interpretations for Deep RL}\label{sec:temporal interpretaions}
 As discussed above, an important observation about the temporal-spatial causal features discovered by TSCI model is that the agent can extract high-level semantic information from consecutive observations of the same object. In other words, there exists an underlying temporal dependence between temporally adjacent frames, which are connected by the same semantic concepts. The goal of this section is to further reveal and understand the role that temporal dependence plays in sequential decision-making. To that end, we first provide counterfactual analysis about temporal dependence by evaluating its impact on the agent's long-term performance, and leverage TSCI model to explain the resulting counterfactual phenomenon. Then, TSCI model is further applied to interpret temporally-extended RL agents and reason about why frame stacking is generally necessary even for the agent that has used a recurrent structure from the point of view of temporal dependence. Finally, we apply the proposed method to provide more downstream interpretations for vision-based RL.

\subsection{Counterfactual Analysis of Temporal Dependence}
 Here an intervention-based approach is proposed to render empirical evidences about the underlying temporal dependence in vision-based RL. Specifically, we intervene on the input state (or temporally extended sequences of frames) to produce counterfactual conditions. Prior work \cite{szegedy2014intriguing} has focused on manipulating the pixel input, but this does not modify the underlying temporal dependence. Instead, we intervene directly on the input state to change semantic concepts related to temporal dependence. For the convenience of manipulation, we intervene on the input state by masking out partial semantic information located on different frames. For example, we denote ``34'' as the scheme where the semantic information in the third and fourth frames are partially masked out, as shown in Figure \ref{fig:example}. In particular, we denote ``None'' as the case without any semantic information masked out.
 \begin{figure}[t]
 	\setlength{\abovecaptionskip}{-0.01cm}
 	\setlength{\belowcaptionskip}{-0.20cm}
 	\begin{center}
 		\includegraphics[width=0.95\linewidth]{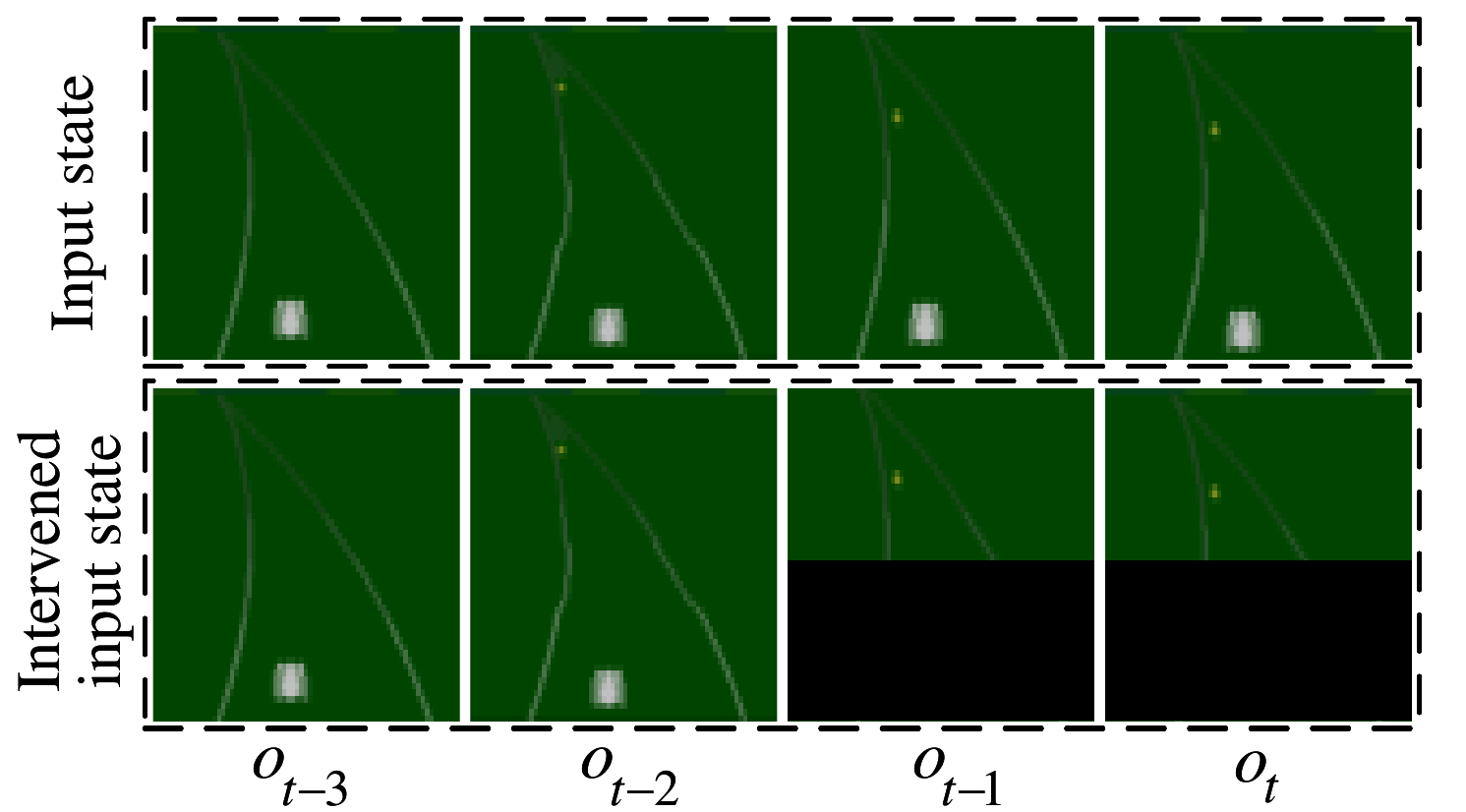}
 	\end{center}
 	\caption{An example for how to intervene on the input state, which consists of the last four consecutive frames $``o_{t-3}\cdots o_t"$.}
 	\label{fig:example}
 \end{figure}

 In order to evaluate the impact of temporal dependence on the agent's long-term performance, we compare the average return of the agent model under different counterfactual conditions. In our implementation, we suppress the underlying temporal dependence between frames to varying degrees by applying different intervene schemes to temporally-extended frames as described above. The agent to be evaluated remains unchanged except for the input state. The empirical results are summarized in Figure \ref{fig:reliability return}. It can be seen that the reduction of temporal dependence results in different degrees of performance degradation. More concretely, first, the performance gradually decreases as the degree to which we intervene on the input state increases (i.e., ``4''$\rightarrow$``34''$\rightarrow$``234''$\rightarrow$``1234''). Second, if the temporal dependence is destroyed completely such as the schemes ``234'' and ``1234'', the agent is close to collapse. Third, although the scheme ``123'' also destroys the temporal dependence, the performance of scheme ``123'' does not collapse dramatically since the fourth frame is the most important frame for making decisions. Fourth, intervening on single previous frame does not lead to obvious performance degradation, such as the schemes ``1'', ``2'' and ``3''. Furthermore, we can conclude that the earlier the frame is observed, the less important it is to the current decision, since the intervention on the current frame leads to larger performance degradation than that on the previous frames, as can be seen from the schemes ``1'', ``2'', ``3'' and ``4''.

 It is worth noting that the semantic information masked out mainly consists of task-irrelevant (or background) information and semantic features related to temporal dependence. Meanwhile, it can be seen from the last column of Figure \ref{fig:evaluation metrics} that the mean normalized return is greater than or equal to 1. In other words, the agent model taking as input only the causal features discovered by TSCI model achieves better performance than that taking as input all available information. Therefore, it can be concluded that the loss of background information does not cause performance degradation, and the performance degradation observed in Figure \ref{fig:reliability return} is mainly attributed to the destruction of semantic features related to temporal dependence.
 \begin{figure}[t]
	\setlength{\abovecaptionskip}{-0.01cm}
	\setlength{\belowcaptionskip}{-0.15cm}
	\begin{center}
		\includegraphics[width=0.95\linewidth]{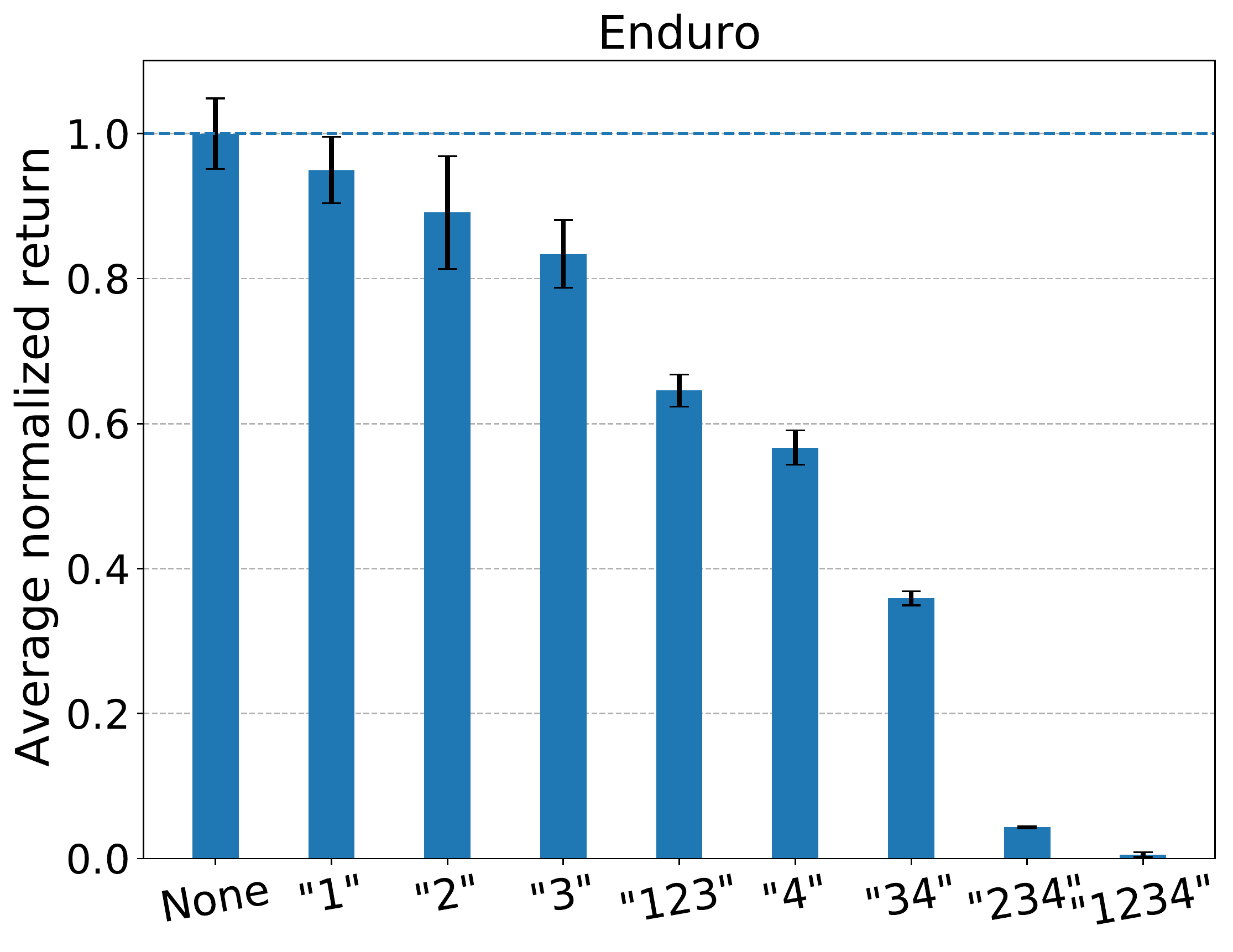}
	\end{center}
	\caption{Performance evaluations of temporal causal features when different intervention schemes are applied to the input state. The agent to be evaluated remains unchanged except for the input state.}
	\label{fig:reliability return}
 \end{figure}

 To further explain and understand how the reduction of temporal dependence affects the agent's decision and performance, we apply the proposed TSCI model to visualize the agent's attention for making decisions under different counterfactual conditions. Specifically, for each intervention scheme, we retrain a separate TSCI model to discover temporal-spatial causal features while the agent to be interpreted remains unchanged. The results are shown in Figure \ref{fig:reliability}, which renders empirical causal reasoning about temporal dependence from two aspects. First, by comparing ``None'' to ``3'' or ``4'', we observe that the agent shifts some of its attention to the previous frames when we only intervene on a single frame of the input state. In other words, the agent can recover temporal dependence partially by extracting high-level semantic features from the unmasked part of input state. Second, when the temporal dependence is destroyed completely such as ``234'' and ``1234'', the agent fails to learn the representation of high-level semantic concepts and is thus prone to collapse. The above observations provide empirical explanations about why the agent is still able to perform well in schemes ``1'', ``2'', and ``3'' but there is obvious performance degradation in schemes ``234'' and ``1234''. Third, under the intervention scheme ``34'', the agent shifts some of its attention to the first two frames, hence the performance does not collapse dramatically like ``234'' and ``1234''. Nevertheless, the scheme ``34'' also leads to obvious performance degradation as can be seen in Figure \ref{fig:reliability return}, since the first two frames are far less important than the last two frames according to the result of ``None'' in Figure \ref{fig:reliability}. In summary, both the last frame and temporal dependence are important to the agent’s performance. While destroying one of them will only lead to varying degrees of performance degradation, destroying all of them is likely to cause a dramatic collapse in performance. In fact, due to the essential role of temporal dependence in vision-based RL, it may be untrustworthy to explain vision-based RL agents with some existing interpretation methods developed exclusively for supervised learning, which does not involve temporal dimension.
 \begin{figure*}[t]
	\setlength{\belowcaptionskip}{-0.10cm}
	\begin{center}
		\includegraphics[width=0.96\linewidth]{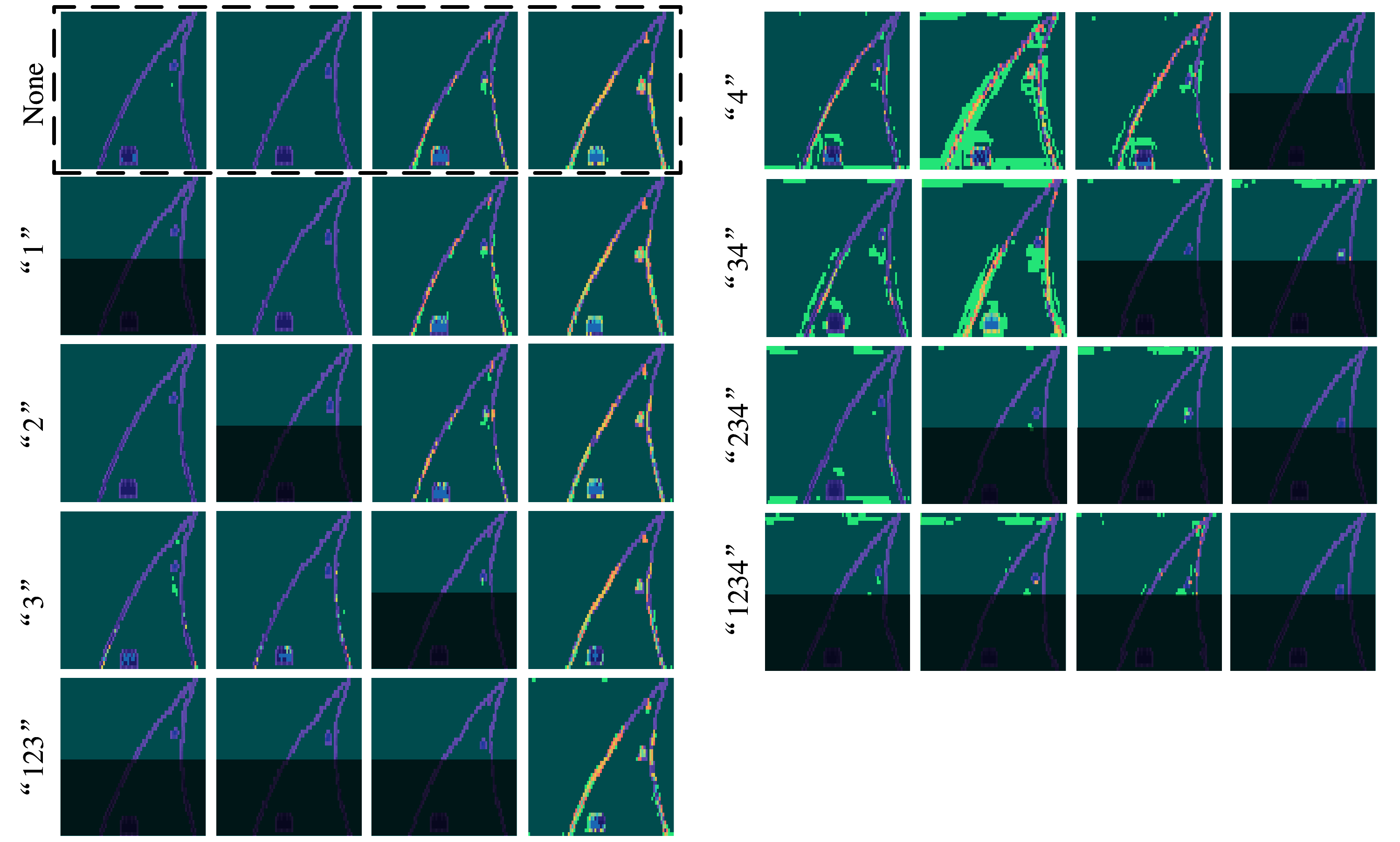}
	\end{center}
	\caption{Visualization of temporal-spatial causal features discovered by the proposed TSCI model when different intervention schemes are applied to intervene the input state. Every row is the same state consisting of four consecutive frames (time goes from left to right). The black areas are left out during the retraining of TSCI model.}
	\label{fig:reliability}
 \end{figure*}

\subsection{Interpreting Temporally-Extended Agents}
 In the context of vision-based RL, the agent is usually devised to have a recurrent structure such as RNNs, especially for partially observable environments. The main motivation of such design is to enhance and utilize the underlying temporal dependence, which has been empirically verified to play a significant role for sequential decision-making in the previous section. In principle, the agent with a recurrent structure should be able to guarantee temporal dependence by extracting high-level information from consecutive observations. However, frame stacking technique is generally needed to obtain temporally-extended input states apart from applying a recurrent structure in many applications. Therefore, a direct question about this is that whether frame stacking is redundant for the agent with a recurrent structure. In other words, is it necessary for temporally-extended agent model to use temporally-extended inputs in vision-based RL?
 \begin{figure}[t]
  \setlength{\abovecaptionskip}{-0.01cm}
  \setlength{\belowcaptionskip}{-0.25cm}
  \begin{center}
   \includegraphics[width=0.95\linewidth]{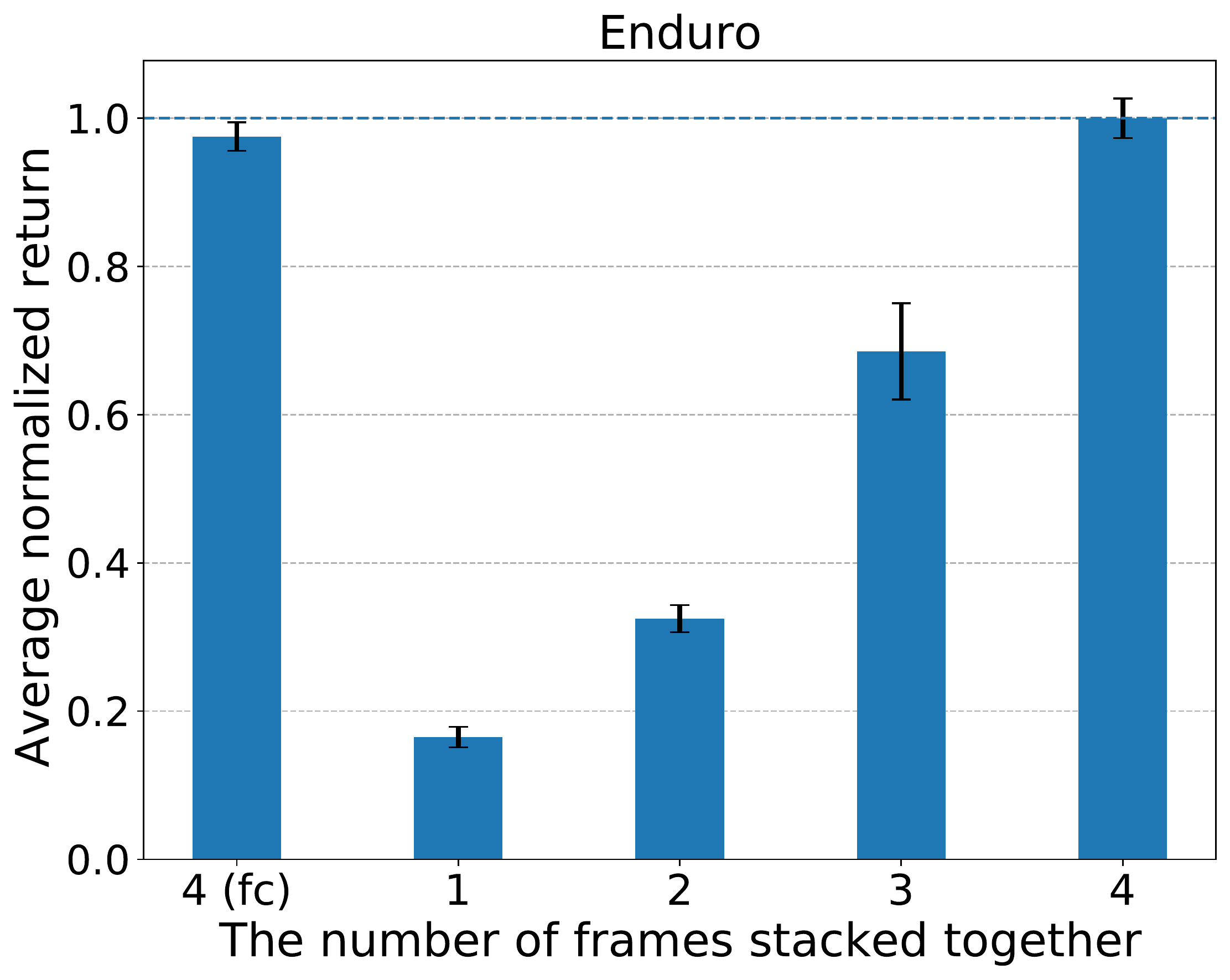}
  \end{center}
  \caption{Performance comparison of five agents that apply different forms of input state or structure. While the first agent applies a full-connected (fc) feedforward structure, the last four agents use the same recurrent structure and their input states are obtained by stacking different number of the last consecutive frames together. All results are averaged across five independent runs.}
  \label{fig:frame return}
 \end{figure}
 \begin{figure*}[t]
  \begin{center}
   \includegraphics[width=0.97\linewidth]{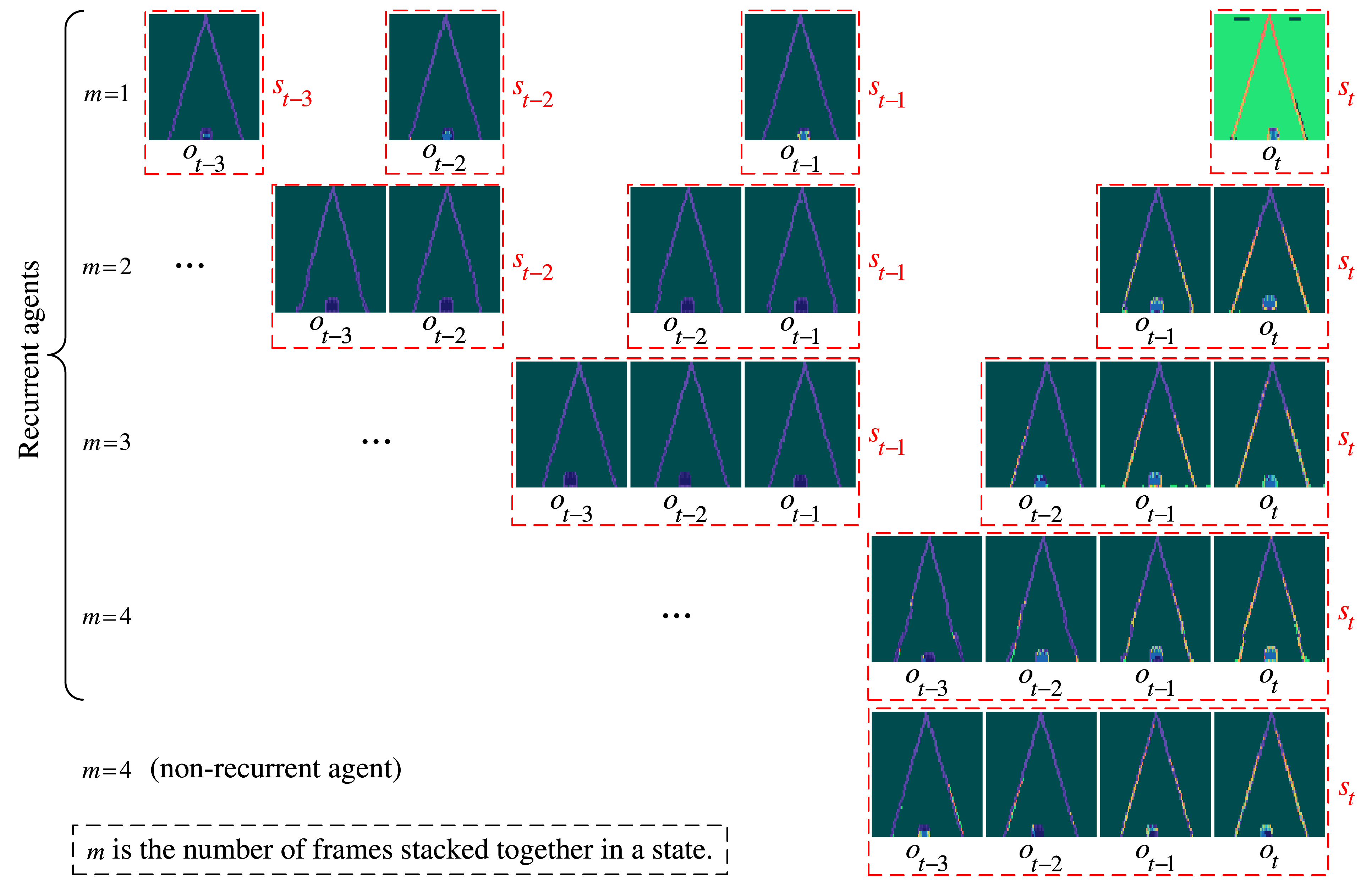}
  \end{center}
  \caption{Comparison between the attentions of four recurrent agent models that apply frame stacking in different ways. Each row visualizes the attention of the agent whose input states $s_k$ are obtained by stacking $m$ consecutive frames $``o_{k-m+1}\cdots o_{k-1}o_k"$ together. Each red dashed rectangular represents the input state $s_k$ at the $k$-th timestep. Here we consider the input states that only include the last four frames $``o_{t-3:t}"$, and thus not all the last four input states $``s_{t-3:t}"$ are visualized due to the limitation of space.}
  \label{fig:frame stacking number}
 \end{figure*}

 To answer the above question, we perform comparative experiments to assess the impacts of both recurrent structure and frame stacking technique on the agent's performance. Specifically, we retrain two groups of agents that apply different forms of input state and network structure. The first group of agents have the same form of input state but different network structures, and the second group of agents use the same recurrent structure but their input states consist of different numbers of the last consecutive frames. The empirical results are shown in Figure \ref{fig:frame return}. It can be seen that the agent performs badly when the input state includes only the current frame and does not include any previous frames (i.e., frame stacking technique is not applied). In contrast, using a full-connected feedforward structure does not lead to significant performance degradation, indicating that a recurrent structure itself may be not enough to capture complete temporal dependence between consecutive frames although it shows remarkable memory ability for low-dimensional sequential data. In addition, the results in Figure \ref{fig:temporal causal features} show that the first two frames are not as important as the last two frames, but we can see that only using the last two or three frames in states causes information loss and thus results in performance degradation as shown by the third and fourth columns of Figure \ref{fig:frame return}. In summary, we can observe that the agent performs better as the input state includes more consecutive frames. Based on these observations, it can be concluded that only using a recurrent structure is not enough for the agent to achieve good performance in vision-based RL, and it is generally necessary to apply frame stacking technique to obtain temporally-extended input states, especially for partially observable RL environments.

 Intuitively, the agent with a recurrent structure is expected to be capable of learning temporal representations for making sequential decisions. To explain and understand why a recurrent agent model still needs temporally-extended inputs in vision-based RL, we leverage the proposed TSCI model to visualize how the recurrent agent's attention changes as the number of frames stacked together in states varies. The empirical results are visualized in Figure \ref{fig:frame stacking number}. It can be seen that the agent fails to attend to causal features located at the previous frames through recurrent structure. For example, the agent mistakenly focuses most of its attention on the current frame while the previous frames are considered almost irrelevant to the current decision as shown in the first row. More concretely, the green background of $s_t$ in the first row indicates that the background in the current frame is mistakenly considered more important than the lane lines in the previous frames, although the lane lines in the current frame are properly considered the most important. As a consequence, the agent cannot build completely the temporal dependence between consecutive frames that are input to the model at different timesteps. Additionally, we can observe that the temporal dependence between the frames that are stacked together in the same state is built successfully, as shown in the last three rows. Therefore, it is generally necessary for temporally-extended agent model to use temporally-extended inputs in vision-based RL from the perspective of temporal dependence.

 In fact, recurrent structures have been shown to be effective in other fields like natural language processing (NLP). Why recurrent structures do not show the same effectiveness in vision-based RL may be attributed to the following reasons. First, there exists information loss in recurrent structures such as GRUs, and the decision of RL agent relies heavily on the temporal dependence between consecutive observations and thus is more sensitive to information loss than NLP, where there are usually only semantic and syntactic connections between different words. Second, the learning of low-dimensional latent vectors of images is generally coupled with the learning of credit assignment, but sparse task rewards do not provide enough signal for the agent to learn what to store in memory and RL agents require self-supervised auxiliary training to learn an abstract, compressed representation of each input frame \cite{fortunato2019generalization, wayne2018unsupervised, lampinen2021towards}.  More potential reasons and their verifications are important research directions for future work.

\subsection{More Downstream Interpretations for Deep RL}
 In addition to the network structure discussed above, there are many other factors that also affect the behavior of RL agents. In this section, we further discuss the impact of different RL algorithms and environments on the causal features of vision-based RL agents. Such an analysis can provide explainable basis for the selection of models in different scenarios.

 \textbf{Case 1: RL algorithm.}
 Here we compare the temporal-spatial causal features of two RL agents that are trained using PPO and Advantage Actor-Critic (A2C) algorithms respectively. The empirical results are provided and visualized in Figure \ref{fig:downstream_algorithm}. It can be seen that there is no remarkable difference between the temporal-spatial causal features discovered by using the proposed TSCI method, but the agent trained using PPO algorithm exhibits better performance than the agent using A2C algorithm in terms of feature saliency and completeness. This observation explains to some extent the reason why PPO generally performs better than A2C on most RL tasks.
 \begin{figure}[t]
  \begin{center}
   \includegraphics[width=0.95\linewidth]{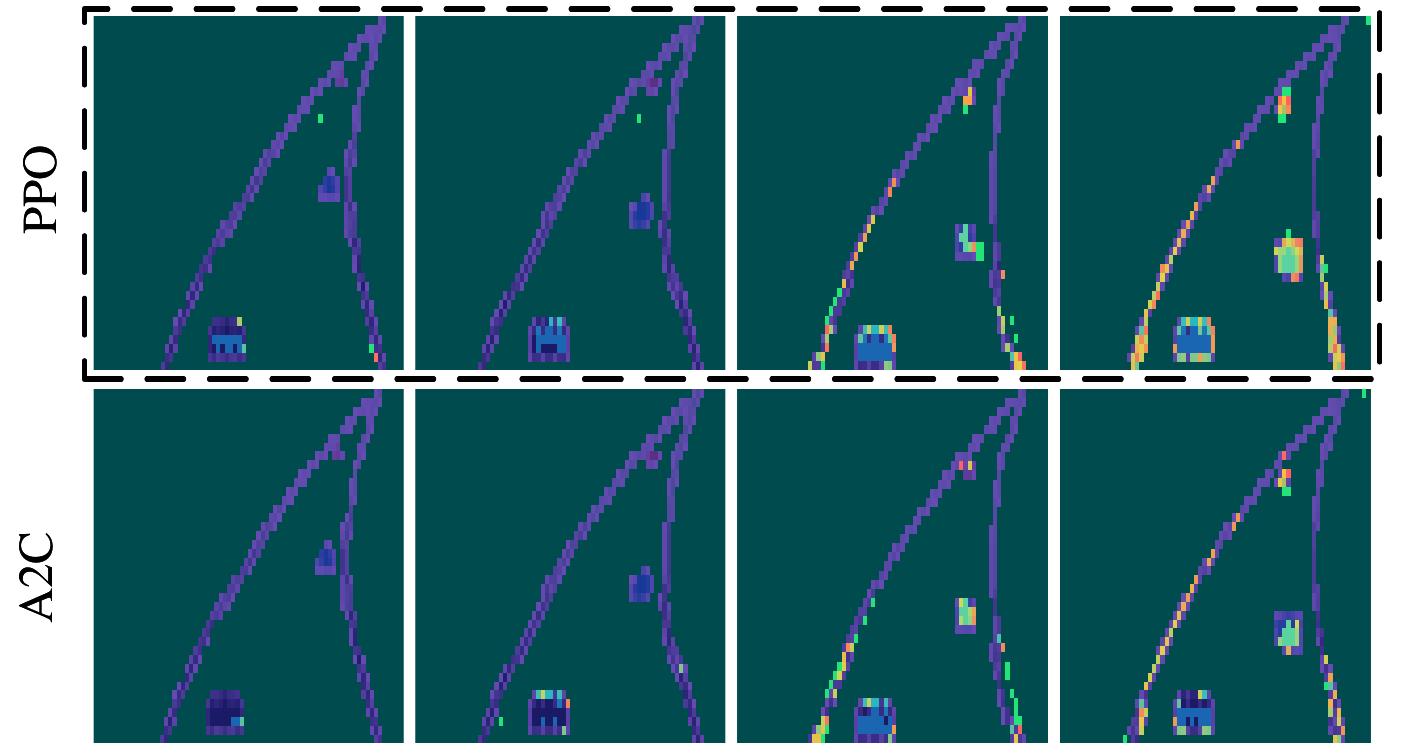}
  \end{center}
  \caption{The temporal-spatial causal features of two RL agents that are trained using PPO and A2C respectively.}
  \label{fig:downstream_algorithm}
 \end{figure}

 \textbf{Case 2: Environment.}
 To evaluate the performance of our method on different environments, we perform comparative experiments on two similar tasks of two different environments, namely Enduro of Arcade Learning Environment and Lane-following of Duckietown \cite{gym_duckietown}. While Arcade Learning Environment builds on top of the Atari 2600 emulator Stella, Duckietown is a self-driving car simulator that builds on top of ROS environment. As shown in Figure \ref{fig:downstream_environment}, our method is able to discover causal features on both Arcade Learning Environment and Duckietown, although they have slight differences in feature saliency.
 \begin{figure}[t]
  \setlength{\belowcaptionskip}{-0.15cm}
  \begin{center}
   \includegraphics[width=0.95\linewidth]{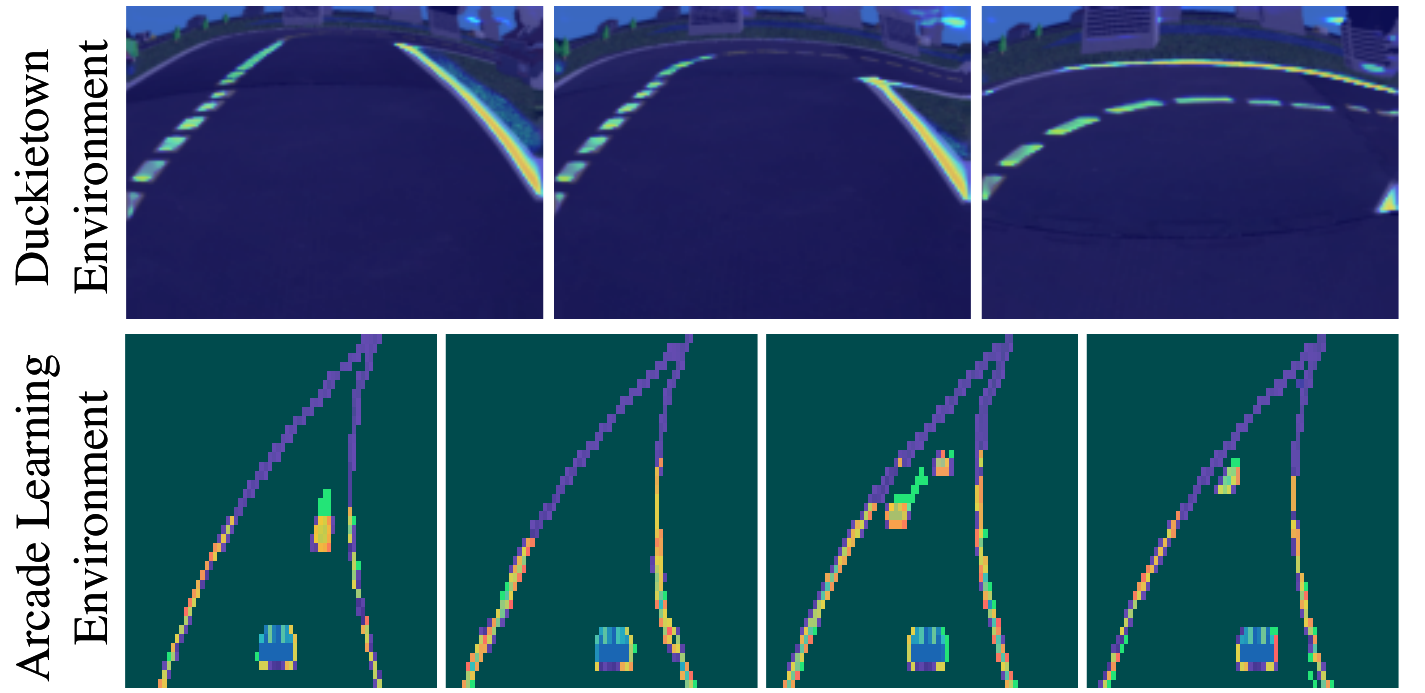}
  \end{center}
  \caption{The temporal-spatial causal features discovered by our method on Duckietown and Arcade Learning environments.}
  \label{fig:downstream_environment}
 \end{figure}

\section{Conclusion}
 We presented a trainable temporal-spatial causal interpretation (TSCI) model for vision-based RL agents. TSCI model is based on the formulation of temporal causality between sequential observations and decisions. To identify temporal-spatial causal features, a separate causal discovery network is employed to learn the temporal causality, which emphasizes the explanation about the agent's long-term behavior rather than single action. This approach has several appealing advantages. First, it is compatible with most RL algorithms and applicable to recurrent agents. Second, we do not require adapting or retaining the agent model to be interpreted. Third, TSCI model can, once trained, be used to discover temporal-spatial causal features in little time. We showed experimentally that TSCI model can produce high-resolution and sharp attention masks to highlight task-relevant temporal-spatial information that constitutes most evidence about how vision-based RL agents make sequential decisions. We also demonstrated that our method is able to provide valuable causal interpretations about the agent's behavior from the temporal-spatial perspective. In summary, this work provides significant insights towards interpretable vision-based RL from a temporal-spatial causal perspective. A more extensive study will be carried out to reason about the temporal dependence between causal features of different observations, and explore how to devise better RL agent structures with strong temporal dependence.

\section*{Acknowledgments}
Wenjie Shi and Gao Huang contribute equally to this work. This work is supported in part by the National Science and Technology Major Project of the Ministry of Science and Technology of China under Grant 2018AAA0100701, the Major Research and Development Project of Guangdong Province under Grant 2020B1111500002, the National Natural Science Foundation of China under Grants 61906106 and 62022048. We would like to thank the reviewers for their valuable comments.


\ifCLASSOPTIONcaptionsoff
  \newpage
\fi



\bibliographystyle{IEEEtran}

\bibliography{reference}

\begin{thebibliography}{10}
\providecommand{\url}[1]{#1}
\csname url@samestyle\endcsname
\providecommand{\newblock}{\relax}
\providecommand{\bibinfo}[2]{#2}
\providecommand{\BIBentrySTDinterwordspacing}{\spaceskip=0pt\relax}
\providecommand{\BIBentryALTinterwordstretchfactor}{4}
\providecommand{\BIBentryALTinterwordspacing}{\spaceskip=\fontdimen2\font plus
\BIBentryALTinterwordstretchfactor\fontdimen3\font minus
  \fontdimen4\font\relax}
\providecommand{\BIBforeignlanguage}[2]{{%
\expandafter\ifx\csname l@#1\endcsname\relax
\typeout{** WARNING: IEEEtran.bst: No hyphenation pattern has been}%
\typeout{** loaded for the language `#1'. Using the pattern for}%
\typeout{** the default language instead.}%
\else
\language=\csname l@#1\endcsname
\fi
#2}}
\providecommand{\BIBdecl}{\relax}
\BIBdecl

\bibitem{mnih2015human}
V.~Mnih, K.~Kavukcuoglu, D.~Silver, A.~A. Rusu, J.~Veness, M.~G. Bellemare,
  A.~Graves, M.~Riedmiller, A.~K. Fidjeland, G.~Ostrovski \emph{et~al.},
  ``Human-level control through deep reinforcement learning,'' \emph{Nature},
  vol. 518, no. 7540, pp. 529--533, 2015.

\bibitem{zhang2020designing}
J.~Zhang, ``Designing optimal dynamic treatment regimes: A causal reinforcement
  learning approach,'' in \emph{International Conference on Machine
  Learning}.\hskip 1em plus 0.5em minus 0.4em\relax PMLR, 2020, pp.
  11\,012--11\,022.

\bibitem{raghu2017continuous}
A.~Raghu, M.~Komorowski, L.~A. Celi, P.~Szolovits, and M.~Ghassemi,
  ``Continuous state-space models for optimal sepsis treatment: a deep
  reinforcement learning approach,'' in \emph{Machine Learning for Healthcare
  Conference}.\hskip 1em plus 0.5em minus 0.4em\relax PMLR, 2017, pp. 147--163.

\bibitem{chia2019machines}
H.~Chia, ``In machines we trust: Are robo-advisers more trustworthy than human
  financial advisers?'' \emph{Law, Tech. \& Hum.}, vol.~1, p. 129, 2019.

\bibitem{cao2019interpretable}
Q.~Cao, X.~Liang, B.~Li, and L.~Lin, ``Interpretable visual question answering
  by reasoning on dependency trees,'' \emph{IEEE Transactions on Pattern
  Analysis and Machine Intelligence}, 2019.

\bibitem{monfort2019moments}
M.~Monfort, A.~Andonian, B.~Zhou, K.~Ramakrishnan, S.~A. Bargal, T.~Yan,
  L.~Brown, Q.~Fan, D.~Gutfreund, C.~Vondrick \emph{et~al.}, ``Moments in time
  dataset: one million videos for event understanding,'' \emph{IEEE
  Transactions on Pattern Analysis and Machine Intelligence}, vol.~42, no.~2,
  pp. 502--508, 2019.

\bibitem{liu2019tabby}
H.~Liu, R.~Wang, S.~Shan, and X.~Chen, ``What is tabby? interpretable model
  decisions by learning attribute-based classification criteria,'' \emph{IEEE
  Transactions on Pattern Analysis and Machine Intelligence}, 2019.

\bibitem{ribeiro2016should}
M.~T. Ribeiro, S.~Singh, and C.~Guestrin, ``"why should i trust you?"
  explaining the predictions of any classifier,'' in \emph{Proceedings of the
  22nd ACM SIGKDD International Conference on Knowledge Discovery and Data
  Mining}, 2016, pp. 1135--1144.

\bibitem{binder2016layer}
A.~Binder, G.~Montavon, S.~Lapuschkin, K.-R. M{\"u}ller, and W.~Samek,
  ``Layer-wise relevance propagation for neural networks with local
  renormalization layers,'' in \emph{International Conference on Artificial
  Neural Networks}.\hskip 1em plus 0.5em minus 0.4em\relax Springer, 2016, pp.
  63--71.

\bibitem{shrikumar2017learning}
A.~Shrikumar, P.~Greenside, and A.~Kundaje, ``Learning important features
  through propagating activation differences,'' in \emph{International
  Conference on Machine Learning}.\hskip 1em plus 0.5em minus 0.4em\relax PMLR,
  2017, pp. 3145--3153.

\bibitem{selvaraju2017grad}
R.~R. Selvaraju, M.~Cogswell, A.~Das, R.~Vedantam, D.~Parikh, and D.~Batra,
  ``Grad-cam: Visual explanations from deep networks via gradient-based
  localization,'' in \emph{Proceedings of the IEEE International Conference on
  Computer Vision}, 2017, pp. 618--626.

\bibitem{lundberg2017unified}
S.~M. Lundberg and S.-I. Lee, ``A unified approach to interpreting model
  predictions,'' in \emph{Advances in Neural Information Processing Systems},
  2017, pp. 4765--4774.

\bibitem{bau2017network}
D.~Bau, B.~Zhou, A.~Khosla, A.~Oliva, and A.~Torralba, ``Network dissection:
  Quantifying interpretability of deep visual representations,'' in
  \emph{Proceedings of the IEEE Conference on Computer Vision and Pattern
  Recognition}, 2017, pp. 6541--6549.

\bibitem{zhou2018interpreting}
B.~Zhou, D.~Bau, A.~Oliva, and A.~Torralba, ``Interpreting deep visual
  representations via network dissection,'' \emph{IEEE Transactions on Pattern
  Analysis and Machine Intelligence}, vol.~41, no.~9, pp. 2131--2145, 2018.

\bibitem{zahavy2016graying}
T.~Zahavy, N.~Ben-Zrihem, and S.~Mannor, ``Graying the black box: Understanding
  dqns,'' in \emph{International Conference on Machine Learning}, 2016, pp.
  1899--1908.

\bibitem{wang2016dueling}
Z.~Wang, T.~Schaul, M.~Hessel, H.~Van~Hasselt, M.~Lanctot, and N.~De~Freitas,
  ``Dueling network architectures for deep reinforcement learning,'' in
  \emph{International Conference on Machine Learning}, vol.~48, 2016, pp.
  1995--2003.

\bibitem{greydanus2018visualizing}
S.~Greydanus, A.~Koul, J.~Dodge, and A.~Fern, ``Visualizing and understanding
  atari agents,'' in \emph{International Conference on Machine Learning}, 2018.

\bibitem{puri2020explain}
N.~Puri, S.~Verma, P.~Gupta, D.~Kayastha, S.~Deshmukh, B.~Krishnamurthy, and
  S.~Singh, ``Explain your move: Understanding agent actions using specific and
  relevant feature attribution,'' in \emph{International Conference on Learning
  Representations}, 2020.

\bibitem{mott2019towards}
A.~Mott, D.~Zoran, M.~Chrzanowski, D.~Wierstra, and D.~Jimenez~Rezende,
  ``Towards interpretable reinforcement learning using attention augmented
  agents,'' in \emph{Advances in Neural Information Processing Systems}, 2019,
  pp. 12\,350--12\,359.

\bibitem{karpathy2015visualizing}
A.~Karpathy, J.~Johnson, and L.~Feifei, ``Visualizing and understanding
  recurrent networks,'' \emph{arXiv preprint arXiv:1506.02078}, 2015.

\bibitem{bargal2018excitation}
S.~A. Bargal, A.~Zunino, D.~Kim, J.~Zhang, V.~Murino, and S.~Sclaroff,
  ``Excitation backprop for rnns,'' in \emph{Proceedings of the IEEE Conference
  on Computer Vision and Pattern Recognition}, 2018, pp. 1440--1449.

\bibitem{granger1969investigating}
C.~W. Granger, ``Investigating causal relations by econometric models and
  cross-spectral methods,'' \emph{Econometrica: journal of the Econometric
  Society}, pp. 424--438, 1969.

\bibitem{arnold2007temporal}
A.~Arnold, Y.~Liu, and N.~Abe, ``Temporal causal modeling with graphical
  granger methods,'' in \emph{Advances in Knowledge Discovery and Data Mining},
  2007, pp. 66--75.

\bibitem{gong2015discovering}
M.~Gong, K.~Zhang, B.~Schoelkopf, D.~Tao, and P.~Geiger, ``Discovering temporal
  causal relations from subsampled data,'' in \emph{International Conference on
  Machine Learning}, 2015, pp. 1898--1906.

\bibitem{schwab2019granger}
P.~Schwab, D.~Miladinovic, and W.~Karlen, ``Granger-causal attentive mixtures
  of experts: Learning important features with neural networks,'' in
  \emph{Proceedings of the AAAI Conference on Artificial Intelligence},
  vol.~33, 2019, pp. 4846--4853.

\bibitem{schwab2019cxplain}
P.~Schwab and W.~Karlen, ``Cxplain: Causal explanations for model
  interpretation under uncertainty,'' in \emph{Advances in Neural Information
  Processing Systems}, 2019, pp. 10\,220--10\,230.

\bibitem{bellemare2013arcade}
M.~G. Bellemare, Y.~Naddaf, J.~Veness, and M.~Bowling, ``The arcade learning
  environment: An evaluation platform for general agents,'' \emph{Journal of
  Artificial Intelligence Research}, vol.~47, pp. 253--279, 2013.

\bibitem{alharin2020reinforcement}
A.~Alharin, T.-N. Doan, and M.~Sartipi, ``Reinforcement learning interpretation
  methods: A survey,'' \emph{IEEE Access}, vol.~8, pp. 171\,058--171\,077,
  2020.

\bibitem{simonyan2014deep}
K.~Simonyan, A.~Vedaldi, and A.~Zisserman, ``Deep inside convolutional
  networks: Visualising image classification models and saliency maps,'' in
  \emph{Workshop Proceedings of the International Conference on Learning
  Representations}, 2014.

\bibitem{sundararajan2017axiomatic}
M.~Sundararajan, A.~Taly, and Q.~Yan, ``Axiomatic attribution for deep
  networks,'' in \emph{International Conference on Machine Learning}, 2017, pp.
  3319--3328.

\bibitem{zhang2018top}
J.~Zhang, S.~A. Bargal, Z.~Lin, J.~Brandt, X.~Shen, and S.~Sclaroff, ``Top-down
  neural attention by excitation backprop,'' \emph{International Journal of
  Computer Vision}, vol. 126, no.~10, pp. 1084--1102, 2018.

\bibitem{ghorbani2019interpretation}
A.~Ghorbani, A.~Abid, and J.~Zou, ``Interpretation of neural networks is
  fragile,'' in \emph{Proceedings of the AAAI Conference on Artificial
  Intelligence}, vol.~33, no.~01, 2019, pp. 3681--3688.

\bibitem{stergiou2019saliency}
A.~Stergiou, G.~Kapidis, G.~Kalliatakis, C.~Chrysoulas, R.~C. Veltkamp, and
  R.~Poppe, ``Saliency tubes: Visual explanations for spatio-temporal
  convolutions,'' in \emph{Proceedings of the IEEE International Conference on
  Image Processing}, 2019, pp. 1830--1834.

\bibitem{fong2017interpretable}
R.~C. Fong and A.~Vedaldi, ``Interpretable explanations of black boxes by
  meaningful perturbation,'' in \emph{Proceedings of the IEEE International
  Conference on Computer Vision}, 2017, pp. 3429--3437.

\bibitem{dabkowski2017real}
P.~Dabkowski and Y.~Gal, ``Real time image saliency for black box
  classifiers,'' in \emph{Advances in Neural Information Processing Systems},
  2017, pp. 6967--6976.

\bibitem{zeiler2014visualizing}
M.~D. Zeiler and R.~Fergus, ``Visualizing and understanding convolutional
  networks,'' in \emph{Proceedings of the European Conference on Computer
  Vision}.\hskip 1em plus 0.5em minus 0.4em\relax Springer, 2014, pp. 818--833.

\bibitem{shapley1953value}
L.~S. Shapley, ``A value for n-person games,'' \emph{Contributions to the
  Theory of Games}, vol.~2, no.~28, pp. 307--317, 1953.

\bibitem{ancona2019explaining}
M.~Ancona, C.~{\"O}ztireli, and M.~Gross, ``Explaining deep neural networks
  with a polynomial time algorithm for shapley values approximation,'' in
  \emph{International Conference on Machine Learning}, 2019.

\bibitem{iyer2018transparency}
R.~R. Iyer, Y.~Li, H.~Li, M.~Lewis, R.~Sundar, and K.~Sycara, ``Transparency
  and explanation in deep reinforcement learning neural networks,'' in
  \emph{Proceedings of the 2018 AAAI/ACM Conference on AI, Ethics, and
  Society}, 2018, pp. 144--150.

\bibitem{atrey2020exploratory}
A.~Atrey, K.~Clary, and D.~Jensen, ``Exploratory not explanatory:
  Counterfactual analysis of saliency maps for deep reinforcement learning,''
  in \emph{International Conference on Learning Representations}, 2020.

\bibitem{wang2020paying}
W.~Wang, J.~Shen, X.~Lu, S.~C. Hoi, and H.~Ling, ``Paying attention to video
  object pattern understanding,'' \emph{IEEE Transactions on Pattern Analysis
  and Machine Intelligence}, 2020.

\bibitem{manchin2019reinforcement}
A.~Manchin, E.~Abbasnejad, and A.~van~den Hengel, ``Reinforcement learning with
  attention that works: A self-supervised approach,'' in \emph{International
  Conference on Neural Information Processing}.\hskip 1em plus 0.5em minus
  0.4em\relax Springer, 2019, pp. 223--230.

\bibitem{nikulin2019free}
D.~Nikulin, A.~Ianina, V.~Aliev, and S.~Nikolenko, ``Free-lunch saliency via
  attention in atari agents,'' in \emph{2019 IEEE/CVF International Conference
  on Computer Vision Workshop}.\hskip 1em plus 0.5em minus 0.4em\relax IEEE,
  2019, pp. 4240--4249.

\bibitem{sorokin2015deep}
I.~Sorokin, A.~Seleznev, M.~Pavlov, A.~Fedorov, and A.~Ignateva, ``Deep
  attention recurrent q-network,'' \emph{arXiv preprint arXiv:1512.01693},
  2015.

\bibitem{annasamy2019towards}
R.~M. Annasamy and K.~Sycara, ``Towards better interpretability in deep
  q-networks,'' in \emph{Proceedings of the AAAI Conference on Artificial
  Intelligence}, vol.~33, 2019, pp. 4561--4569.

\bibitem{choi2017multi}
J.~Choi, B.-J. Lee, and B.-T. Zhang, ``Multi-focus attention network for
  efficient deep reinforcement learning,'' in \emph{Workshops at the
  Thirty-First AAAI Conference on Artificial Intelligence}, 2017.

\bibitem{madumal2020explainable}
P.~Madumal, T.~Miller, L.~Sonenberg, and F.~Vetere, ``Explainable reinforcement
  learning through a causal lens,'' in \emph{Proceedings of the AAAI Conference
  on Artificial Intelligence}, 2020.

\bibitem{bastani2018verifiable}
O.~Bastani, Y.~Pu, and A.~Solar-Lezama, ``Verifiable reinforcement learning via
  policy extraction,'' in \emph{Advances in Neural Information Processing
  Systems}, 2018, pp. 2494--2504.

\bibitem{zhang2020atari}
R.~Zhang, Z.~Liu, L.~Guan, L.~Zhang, M.~M. Hayhoe, and D.~H. Ballard,
  ``Atari-head: Atari human eye-tracking and demonstration dataset,'' in
  \emph{Proceedings of the AAAI Conference on Artificial Intelligence}, 2020.

\bibitem{shi2020self}
W.~Shi, G.~Huang, S.~Song, Z.~Wang, T.~Lin, and C.~Wu, ``Self-supervised
  discovering of interpretable features for reinforcement learning,''
  \emph{IEEE Transactions on Pattern Analysis and Machine Intelligence}, 2020.

\bibitem{eichler2006graphical}
M.~Eichler, ``Graphical modelling of multivariate time series with latent
  variables,'' \emph{Preprint, Universiteit Maastricht}, 2006.

\bibitem{drton2008sinful}
M.~Drton and M.~D. Perlman, ``A sinful approach to gaussian graphical model
  selection,'' \emph{Journal of Statistical Planning and Inference}, vol. 138,
  no.~4, pp. 1179--1200, 2008.

\bibitem{valdes2005estimating}
P.~A. Vald{\'e}s-Sosa, J.~M. S{\'a}nchez-Bornot, A.~Lage-Castellanos,
  M.~Vega-Hern{\'a}ndez, J.~Bosch-Bayard, L.~Melie-Garc{\'\i}a, and
  E.~Canales-Rodr{\'\i}guez, ``Estimating brain functional connectivity with
  sparse multivariate autoregression,'' \emph{Philosophical Transactions of the
  Royal Society B: Biological Sciences}, vol. 360, no. 1457, pp. 969--981,
  2005.

\bibitem{opgen2007learning}
R.~Opgen-Rhein and K.~Strimmer, ``Learning causal networks from systems biology
  time course data: an effective model selection procedure for the vector
  autoregressive process,'' \emph{BMC bioinformatics}, vol.~8, no.~2, pp. 1--8,
  2007.

\bibitem{wiering2012reinforcement}
M.~Wiering and M.~Van~Otterlo, ``Reinforcement learning,'' \emph{Adaptation,
  Learning, and Optimization}, vol.~12, p.~3, 2012.

\bibitem{schulman2017proximal}
J.~Schulman, F.~Wolski, P.~Dhariwal, A.~Radford, and O.~Klimov, ``Proximal
  policy optimization algorithms,'' \emph{arXiv preprint arXiv:1707.06347},
  2017.

\bibitem{arjovsky2017wasserstein}
M.~Arjovsky, S.~Chintala, and L.~Bottou, ``Wasserstein generative adversarial
  networks,'' in \emph{International Conference on Machine Learning}.\hskip 1em
  plus 0.5em minus 0.4em\relax PMLR, 2017, pp. 214--223.

\bibitem{todorov2012mujoco}
E.~Todorov, T.~Erez, and Y.~Tassa, ``Mujoco: A physics engine for model-based
  control,'' in \emph{2012 IEEE/RSJ International Conference on Intelligent
  Robots and Systems}, 2012, pp. 5026--5033.

\bibitem{lecun2015deep}
Y.~LeCun, Y.~Bengio, and G.~Hinton, ``Deep learning,'' \emph{Nature}, vol. 521,
  no. 7553, pp. 436--444, 2015.

\bibitem{cho2014properties}
K.~Cho, B.~Van~Merri{\"e}nboer, D.~Bahdanau, and Y.~Bengio, ``On the properties
  of neural machine translation: Encoder-decoder approaches,'' in
  \emph{Proceedings of SSST@EMNLP 2014, Eighth Workshop on Syntax, Semantics
  and Structure in Statistical Translation}, 2014.

\bibitem{ronneberger2015u}
O.~Ronneberger, P.~Fischer, and T.~Brox, ``U-net: Convolutional networks for
  biomedical image segmentation,'' in \emph{International Conference on Medical
  Image Computing and Computer-Assisted Intervention}.\hskip 1em plus 0.5em
  minus 0.4em\relax Springer, 2015, pp. 234--241.

\bibitem{lin2017refinenet}
G.~Lin, A.~Milan, C.~Shen, and I.~Reid, ``Refinenet: Multi-path refinement
  networks for high-resolution semantic segmentation,'' in \emph{Proceedings of
  the IEEE Conference on Computer Vision and Pattern Recognition}, 2017, pp.
  1925--1934.

\bibitem{mayer2016large}
N.~Mayer, E.~Ilg, P.~Hausser, P.~Fischer, D.~Cremers, A.~Dosovitskiy, and
  T.~Brox, ``A large dataset to train convolutional networks for disparity,
  optical flow, and scene flow estimation,'' in \emph{Proceedings of the IEEE
  Conference on Computer Vision and Pattern Recognition}, 2016, pp. 4040--4048.

\bibitem{pomerleau1991efficient}
D.~A. Pomerleau, ``Efficient training of artificial neural networks for
  autonomous navigation,'' \emph{Neural Computation}, vol.~3, no.~1, pp.
  88--97, 1991.

\bibitem{szegedy2014intriguing}
C.~Szegedy, W.~Zaremba, I.~Sutskever, J.~Bruna, D.~Erhan, I.~Goodfellow, and
  R.~Fergus, ``Intriguing properties of neural networks,'' in
  \emph{International Conference on Learning Representations}, 2014.

\bibitem{fortunato2019generalization}
M.~Fortunato, M.~Tan, R.~Faulkner, S.~Hansen, A.~Puigdom{\`e}nech~Badia,
  G.~Buttimore, C.~Deck, J.~Z. Leibo, and C.~Blundell, ``Generalization of
  reinforcement learners with working and episodic memory,'' in \emph{Advances
  in Neural Information Processing Systems}, vol.~32, 2019, pp.
  12\,469--12\,478.

\bibitem{wayne2018unsupervised}
G.~Wayne, C.-C. Hung, D.~Amos, M.~Mirza, A.~Ahuja, A.~Grabska-Barwinska,
  J.~Rae, P.~Mirowski, J.~Z. Leibo, A.~Santoro \emph{et~al.}, ``Unsupervised
  predictive memory in a goal-directed agent,'' \emph{arXiv preprint
  arXiv:1803.10760}, 2018.

\bibitem{lampinen2021towards}
A.~K. Lampinen, S.~C. Chan, A.~Banino, and F.~Hill, ``Towards mental time
  travel: a hierarchical memory for reinforcement learning agents,'' in
  \emph{Advances in Neural Information Processing Systems}, 2021.

\bibitem{gym_duckietown}
M.~Chevalier-Boisvert, F.~Golemo, Y.~Cao, B.~Mehta, and L.~Paull, ``Duckietown
  environments for openai gym,''
  \url{https://github.com/duckietown/gym-duckietown}, 2018.

\end{thebibliography}
\end{document}